\documentclass[10pt,twocolumn,letterpaper]{article}

\usepackage[pagenumbers]{cvpr} 

\definecolor{cvprblue}{rgb}{0.21,0.49,0.74}
\usepackage[pagebackref,breaklinks,colorlinks,allcolors=cvprblue]{hyperref}
    
\usepackage{multirow}
\usepackage{amsfonts}
\usepackage{amsmath}
\usepackage{amsthm}
\usepackage[normalem]{ulem}
\usepackage{algorithm}
\usepackage{arydshln}
\usepackage{algpseudocode}

\newtheorem{proposition}{Proposition}

\newcommand{\grantsponsor}[1]{#1}
\newcommand{\grantnumber}[2]{$^{#1}$#2}  

\def\paperID{34875} 
\def\confName{CVPR}
\def\confYear{2026}

\title{Adaptive Auxiliary Prompt Blending for Target-Faithful Diffusion Generation}

\author{Kwanyoung Lee \quad SeungJu Cha \quad Yebin Ahn \quad Hyunwoo Oh \quad Sungho Koh \quad Dong-Jin Kim\\
Hanyang University, South Korea\\
{\tt\small \{mobled37, sju9020, yeshine, komjii, ksh000906, djdkim\}@hanyang.ac.kr}
}

\begin{document}
\twocolumn[{
\renewcommand\twocolumn[1][]{#1}
\maketitle
}]
\begin{abstract}

Diffusion-based text-to-image (T2I) models have made remarkable progress in generating photorealistic and semantically rich images.
However, when the target concepts lie in low-density regions of the training distribution, these models often produce semantically misaligned or structurally inconsistent results.
This limitation arises from the long-tailed nature of text-image datasets, where rare concepts or editing instructions are underrepresented.
To address this, we introduce Adaptive Auxiliary Prompt Blending (AAPB) — a unified framework that stabilizes the diffusion process in low-density regions.
AAPB leverages auxiliary anchor prompts to provide semantic support in rare concept generation and structural support in image editing, ensuring faithful guidance toward the target prompt.
Unlike prior heuristic prompt alternation methods, AAPB derives a closed-form adaptive coefficient that optimally balances the influence between the auxiliary anchor and the target prompt at each diffusion step.
Grounded in Tweedie's identity, our formulation provides a principled and training-free framework for adaptive prompt blending, ensuring stable and target-faithful generation.
We demonstrate the effectiveness of adaptive interpolation over fixed interpolation through controlled experiments and empirically show consistent improvements on the RareBench and FlowEdit datasets, achieving superior semantic accuracy and structural fidelity compared to prior training-free baselines. 
Code is available \href{https://github.com/mobled37/AAPB}{here}.

\end{abstract}    
\section{Introduction}

Text-to-image (T2I) diffusion models have achieved remarkable progress in generating photorealistic and semantically rich images from natural language descriptions~\cite{rombach2022high, podell2023sdxl, esser2024scaling}.
Recent advances in large-scale training have yielded impressive fidelity across diverse visual domains and styles~\cite{blackforestlabs2024flux, chen2023pixart, zhang2024itercomp}.
However, a persistent challenge remains: when the target concepts lie in \textit{low-density regions} of the training distribution, these models often fail to produce faithful or structurally consistent results~\cite{samuel2024generating}.

\begin{figure}
    \centering
    \includegraphics[width=1\linewidth]{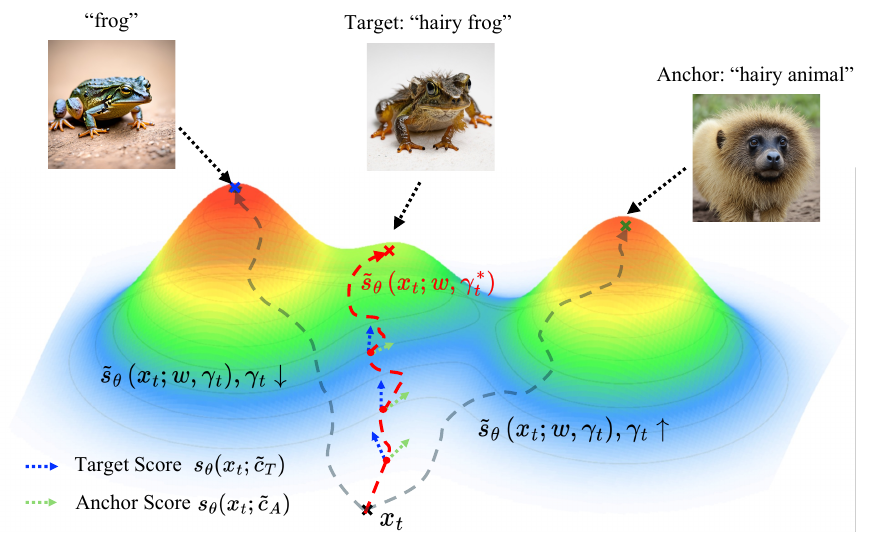}
    \caption{When the target concept lies in a low-density region, the generated samples tend to drift toward semantically dominant, high-density concepts~\cite{um2023don} in the learned score space, resulting in the suppression of rare or compositional attributes. Our proposed adaptive coefficient $\gamma_t^{*}$, derived from auxiliary prompt blending between $\tilde{c}_T$ and its anchor $\tilde{c}_A$, dynamically corrects this bias and produces target-faithful results. Unlike a fixed coefficient $\gamma_t$, the adaptive $\gamma_t^{*}$ adjusts per timestep to maintain a target-aligned denoising trajectory.}
    \label{fig:teaser}
\end{figure}

This problem stems from the inherently long-tailed nature of text–image datasets, where common concepts (e.g., ``frog'', ``cat'') dominate, while rare or compositional ones (e.g., ``hairy frog'', ``origami cat'') appear sparsely or not at all~\cite{park2024rare, samuel2024generating}.
Consequently, the learned score function becomes under-constrained in these regions, causing generations to drift toward semantically dominant concepts~\cite{um2023don}.
This issue manifests across two related settings: (1) in \textit{rare concept generation}~\cite{park2024rare, samuel2024generating}, models fail to express target-specific attributes, and (2) in \textit{image editing}~\cite{kulikov2024flowedit, yang2024text, rout2024semantic, meng2021sdedit}, models struggle to preserve the original structure under uncommon or complex edits.

To mitigate these failures, prior works leverage auxiliary prompts as stabilizing anchors—using semantically related frequent prompts for rare concept generation~\cite{park2024rare} and the original source prompt for image editing~\cite{kulikov2024flowedit}.
However, the key challenge lies in balancing the influence between the target prompt and its auxiliary anchor throughout the diffusion process.
As shown in Fig.~\ref{fig:teaser}, an improper balance leads to two failure modes: over-reliance on the anchor suppresses the target semantics, whereas insufficient anchoring yields unstable generations that deviate from the intended concept.
This motivates the need for a principled formulation of how auxiliary and target prompts should be blended across timesteps.

To address this, we introduce \textbf{Adaptive Auxiliary Prompt Blending (AAPB)}, a training-free framework that stabilizes diffusion generation by adaptively modulating the contribution of the auxiliary anchor throughout the denoising process.
Prior methods, such as R2F~\cite{park2024rare}, address this trade-off by using fixed and heuristic prompt schedules that vary across prompts and tasks.
In contrast, we derive a \textit{closed-form adaptive coefficient} that dynamically balances the contributions of the target and anchor at each diffusion step, ensuring both semantic precision and structural stability.
Building on Tweedie's identity~\cite{Robbins1992}, our goal is to align the posterior mean of the blended denoised estimate toward that of the target prompt.
Then, we derive a closed-form adaptive coefficient that minimizes the distance between the posterior mean, ensuring the blended denoising trajectory remains aligned with the target concept rather than drifting toward high-density modes.
This framework naturally unifies rare concept generation and image editing, both of which involve target concepts that lie in low-density regions, under a single adaptive optimization principle.

We build intuition for our approach through a controlled toy example and a supporting proposition showing that the adaptive coefficient achieves a lower squared 2-Wasserstein distance than fixed interpolation under idealized conditions, followed by empirical validation on real-world tasks.
We validate AAPB across two distinct tasks—rare concept generation on RareBench~\cite{park2024rare} and image editing on FlowEdit~\cite{kulikov2024flowedit} dataset—and observe substantial improvements over prior training-free baselines, notably enhancing semantic accuracy for rare concepts and delivering more structure-preserving edits.
The main contributions of this work are summarized as follows:

\begin{enumerate}
    \item We introduce \textbf{AAPB}, a unified and training-free approach for target-faithful diffusion generation in low-density regions.
    \item We derive a \textbf{closed-form adaptive coefficient} that replaces heuristic anchor scheduling by optimally balancing target and anchor guidance through Tweedie-based score alignment.
    \item We provide theoretical insight via a controlled toy example and proposition showing that adaptive blending yields a lower squared 2-Wasserstein distance than fixed interpolation, and verify this behavior empirically.
    \item We demonstrate strong and consistent gains on RareBench and FlowEdit datasets, improving semantic accuracy and structural fidelity across both rare concept generation and image editing tasks.
\end{enumerate}

\section{Related Works} 

\noindent \textbf{Rare Concept Generation.}
While large-scale pre-trained diffusion models~\cite{rombach2022high, podell2023sdxl, esser2024scaling} demonstrate strong visual quality, they still struggle to synthesize semantically rare or visually underrepresented concepts. 
SynGen~\cite{rassin2023linguistic} tackles fine-grained misalignment via cross-attention modulation but remains confined to frequent concepts seen during training. 
LLM-grounded diffusion frameworks~\cite{lian2023llm, yang2024mastering, hu2024ella} improve prompt understanding and compositional reasoning, yet still fail to generalize to tail concepts, causing unstable or inconsistent generations.
Recent works, SeedSelect~\cite{samuel2024generating} and R2F~\cite{park2024rare}, leverage head or frequent samples to improve rare concept synthesis—SeedSelect in the image space by retrieving optimal noise seeds, and R2F in the semantic space via LLM-based frequent prompts. 
However, both rely on external references or fixed LLM schedules, limiting zero-shot generalization. 
Specifically, R2F alternates between target and anchor at the prompt level, preventing continuous score-space interpolation and per-timestep adaptivity. 
In contrast, we operate directly in score space, enabling closed-form adaptive coefficients to balance contributions per timestep.

\noindent \textbf{Image Editing.}
Training-free editing methods such as SDEdit~\cite{meng2021sdedit} perturb an input image with Gaussian noise and denoise it via diffusion priors, but often suffer from structural degradation.
Deterministic ODE inversion~\cite{rout2024semantic} improves reconstruction by reversing diffusion trajectories, yet demands fine discretization or prompt optimization for controllable edits.
iRFDS~\cite{yang2024text} leverages rectified flow distillation to achieve inversion and editing without stochastic noise, but struggles with local attribute consistency.
FlowEdit~\cite{kulikov2024flowedit} employs an inversion-free ODE that effectively maintains structural consistency, although it may lose original content structure under stronger or more complex edit instructions.
Building on this line, our AAPB framework targets \textit{zero-shot} semantic editing with strong structural fidelity.
We adopt the inversion-free formulation of FlowEdit as our foundation and introduce adaptive auxiliary prompt blending in the score space to achieve prompt-faithful, structure-preserving generation without retraining or optimization.

\section{Backgrounds}

Before delving into details of our work, we briefly review several key concepts related to diffusion-based generative models. 
We provide an overview of diffusion models with an emphasis on their score-based formulation~\cite{song2020score}, which serves as the foundation of our framework.
We next review two key components that ground our method: Tweedie’s denoising formula~\cite{Robbins1992} and the classifier-free guidance (CFG) mechanism~\cite{ho2022classifier}, which together form the theoretical basis of adaptive prompt blending.

\subsection{Diffusion Models}

Diffusion-based generative models~\cite{ho2020denoising, song2019generative, song2020score} 
define a forward noising process that gradually perturbs clean data $x_0 \sim p_{\text{data}}$ into Gaussian noise 
through a Markov chain. At timestep $t$, the forward process admits the closed form
\begin{equation}
q_\alpha(x_t \mid x_0) = \mathcal{N}\!\left(x_t; \sqrt{\alpha_t}x_0, (1-\alpha_t)I \right),
\label{eq:forward}
\end{equation}
where $\alpha_t = \prod_{s=1}^t (1-\beta_s)$ with $\{\beta_s\}_{s=1}^T$ the variance schedule.  
As $t \to T$, $x_T$ approaches the standard Gaussian prior, $x_T \sim \mathcal{N}(0,I)$.

The generative process corresponds to reversing this Markov chain, which requires access to the score function $\nabla_{x_t} \log q_\alpha(x_t \mid x_0)$. 
In practice, one trains a neural score network $s_\theta(x_t, t)$ to approximate the true score. 
Following denoising score matching~\cite{song2020score}, the training objective is a weighted sum of per-timestep score errors. 
For simplicity, we write $s_\theta(x_t)$ for $s_\theta(x_t, t)$.

For notational convenience, we define
\begin{equation}
\mathbb{E}_t[\cdot] := \mathbb{E}_{x_0 \sim p_{\text{data}},\, x_t \sim q_t(\cdot \mid x_0)}[\cdot],
\end{equation}
where $q_t(x_t \mid x_0)$ denotes the forward noising distribution at timestep $t$.
With this shorthand, the objective becomes
\begin{equation}
\min_\theta \; \sum_{t=1}^T w_t \, 
\mathbb{E}_t \Big[\, \| s_\theta(x_t) - \nabla_{x_t} \log q_t(x_t \mid x_0) \|_2^2 \,\Big],
\label{eq:score-matching}
\end{equation}
where $w_t = 1-\alpha_t$ is a time-dependent weighting factor. With sufficient data and model capacity, the optimal solution $s_\theta^\ast$ recovers the exact perturbed score function $s_\theta^\ast(x_t) = \nabla_{x_t}\log q_\alpha(x_t)$ almost everywhere~\cite{um2023don}.

\subsection{Tweedie’s Formula in Diffusion Models}
\label{sec:background_tweedie}

\emph{Tweedie’s formula}~\cite{Robbins1992} defines the Bayes-optimal denoiser under Gaussian corruption as
\begin{equation}
\mathbb{E}[x_0 \mid x_t] = x_t + (1-\alpha_t) \nabla_{x_t}\log p(x_t),
\label{eq:tweedie}
\end{equation}
which represents the posterior mean—the minimum-mean-square-error (MMSE) estimate of the clean sample given its noisy observation.
Given the forward process $q_\alpha(x_t|x_0)$ in Eq.~\eqref{eq:forward}, the true score can be expressed as  
\begin{equation}
s^*(x_t) = \mathbb{E}_{q_\alpha(x_0|x_t)}[\nabla_{x_t}\log q_\alpha(x_t|x_0)],
\end{equation}
where the expectation is taken over $q_\alpha(x_0|x_t) \propto q_\alpha(x_t|x_0)p(x_0)$.  
As Um et al.~\cite{um2023don} point out, this expectation inherently biases the posterior mean toward majority features of $p(x_0)$ because frequent samples dominate the conditional averaging process.
Consequently, denoised estimates $\mathbb{E}[x_0|x_t]$ tend to drift toward high-density regions of the data distribution.

\subsection{Classifier-Free Guidance (CFG)}
Let $t$ denote the diffusion timestep and $s_\theta(x_t)$ the unconditional score estimator.
\emph{Classifier-Free Guidance} (CFG)~\cite{ho2022classifier} jointly trains conditional and unconditional score estimators within a shared network, where the unconditional branch is obtained by randomly dropping the condition token during training. During inference, the two scores are linearly combined as
\begin{equation}
\tilde{s}_{\theta}^{\text{CFG}}(x_t, c; w)
=
ws_{\theta}(x_t, c)
+
(1-w)s_{\theta}(x_t),
\label{eq:cfg-classic}
\end{equation}
where $w$ controls the strength of conditional guidance. 
Increasing $w$ biases the generation toward well-learned, high-likelihood regions of the conditional distribution $p(x|c)$, thereby improving semantic alignment~\cite{karras2024guiding, um2023don}.

\begin{figure*}[ht!]
    \centering
    \includegraphics[width=1\linewidth]{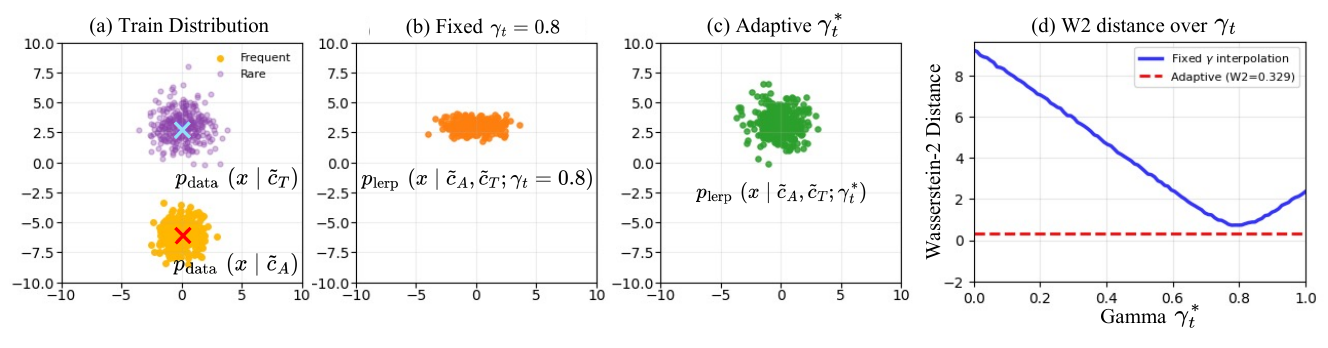}
    \caption{
    \textbf{Toy Example of the target concept generation.} (a) Training distributions: frequent samples $\mathcal{N}((0, -6), I)$ (orange) and rare-prior samples $\mathcal{N}((0, 3), 1.5I)$ (purple) with mean positions marked by crosses; (b) Generated samples using fixed linear interpolation $p_{\text{lerp}}(x|\tilde{c}_A, \tilde{c}_T; \gamma_t = 0.8)$; (c) Generated samples using adaptive interpolation $p(x| \tilde{c}_A,\tilde{c}_T; \gamma^*_t)$; (d) 2-Wasserstein distance between generated distributions and the target $\mathcal{N}((0, 3), 1.5I)$ as a function of interpolation parameter $\gamma_t$ (blue line). The distance curve shows that while fixed interpolation achieves a minimum distance around $\gamma_t \approx 0.8$, the adaptive method (red dashed line) consistently outperforms any fixed interpolation choice.}
    \label{fig:toyexample}
\end{figure*}

\section{Method}

Our goal is to enhance low-density diffusion generation by introducing an adaptive blending mechanism that blends the target prompt with an auxiliary anchor prompt. 
We refer to this unified framework as \textbf{Adaptive Auxiliary Prompt Blending (AAPB)}. 
AAPB generalizes classifier-free guidance by adaptively modulating the contribution of an auxiliary anchor—either a semantically aligned frequent prompt (in rare concept generation) or the unedited source prompt (in image editing)—at each diffusion step. 
At the core of AAPB is a per-timestep adaptive coefficient $\gamma_t^\ast$, which governs the balance between target fidelity and anchor stability.
We denote the target prompt by $\tilde{c}_T$ and the auxiliary anchor prompt by $\tilde{c}_A$. 
In rare concept generation, $(\tilde{c}_T,\tilde{c}_A)$ represent the rare target concept and its frequent anchor~\cite{park2024rare}.
In image editing, $(\tilde{c}_T,\tilde{c}_A)$ correspond to the edited target prompt and the original source prompt, respectively.

Recall that standard CFG~\cite{ho2022classifier} combines unconditional and conditional scores as in Eq.~\eqref{eq:cfg-classic}. 
We generalize this by redefining the conditional score as a linear blending between the target and auxiliary anchor scores:
\begin{equation}
s_\theta(x_t,c) = (1-\gamma_t)\,s_\theta(x_t,\tilde{c}_T) + \gamma_t\,s_\theta(x_t,\tilde{c}_A),
\label{eq:interpolated_score}
\end{equation}
where $\gamma_t$ is the blending coefficient.
Substituting Eq.~\eqref{eq:interpolated_score} into Eq.~\eqref{eq:cfg-classic} yields our blended guidance formulation:
\begin{equation}
\begin{aligned}
&\tilde{s}_\theta(x_t; w, \gamma_t) 
= s_\theta(x_t) \\&+ w\big((1-\gamma_t)\,s_\theta(x_t,\tilde{c}_T) 
+ \gamma_t\,s_\theta(x_t,\tilde{c}_A) - s_\theta(x_t)\big),
\label{eq:blended_score}
\end{aligned}
\end{equation}
where $s_\theta(x_t)$ is the unconditional score, $s_\theta(x_t,\tilde{c}_T)$ the target-conditioned score, and $s_\theta(x_t,\tilde{c}_A)$ the auxiliary anchor-conditioned score.

\subsection{Posterior Mean Alignment}

Building on the analysis in Sec.~\ref{sec:background_tweedie}, we leverage the drift of Tweedie’s posterior mean toward high-density regions of the data distribution as the key insight for our loss formulation.
R2F~\cite{park2024rare} has shown that leveraging a semantically related frequent concept as an auxiliary anchor prompt during generation reduces the gap between generated and real data distributions.
However, because the blended score linearly combines target and anchor guidance, we conjecture that its corresponding posterior mean also shifts toward the anchor’s majority region while remaining distant from the target data distribution. 
To explicitly reduce this bias, we define a regularization loss based on the posterior-mean distance between the blended denoiser $\tilde{\mu}_\theta$ and the target denoiser $\mu_\theta^T$:

\begin{equation}
\mathcal{L}_{\text{reg}}(\gamma_t) = \|\tilde{\mu}_\theta(x_t;w,\gamma_t)-\mu_\theta^T(x_t)\|_2^2.
\label{eq:image_domain_loss}
\end{equation}

To make this objective tractable, we apply Tweedie's formula (Eq.~\eqref{eq:tweedie}) to express both $\tilde{\mu}_\theta$ and $\mu_\theta^T$ in terms of their respective score functions~\cite{luo2022understanding}:
\begin{equation}
\mu_{\theta}\left(x_t\right)=\frac{1}{\sqrt{\alpha_t}} \left( x_t - \sqrt{1-\alpha_t} \cdot \epsilon_{\theta}\left(x_t\right) \right),
\end{equation}
where the predicted noise $\epsilon_\theta(x_t)$ is directly related to the score via the identity $s_\theta(x_t) = -\epsilon_\theta(x_t)/\sqrt{1-\alpha_t}$. 
More critically, this implies that the squared error between any two score functions is directly proportional to the squared error between their corresponding denoisers:
\begin{equation}
\begin{aligned}
\|\tilde{\mu}_\theta(x_t;&w,\gamma_t)-\mu_\theta^T(x_t)\|_2^2 
\\ &= \frac{(1-\alpha_t)^2}{\alpha_t}\|\tilde{s}_\theta(x_t;w,\gamma_t)-s_\theta(x_t,\tilde{c}_T)\|_2^2,
\label{eq:tweedie-identity}
\end{aligned}
\end{equation}
where $\tilde{c}_T$ is the target prompt.

This equivalence is central to our framework: it establishes that optimizing for target-faithful generation in image space is equivalent to minimizing score-space error.
Since the pre-factor $\alpha_t$ depends only on the noise schedule and does not affect the optimal $\gamma_t$, we can directly optimize:
\begin{equation}
\mathcal{L}(\gamma_t)=\|\tilde{s}_\theta(x_t;w,\gamma_t)-s_\theta(x_t,\tilde{c}_T)\|_2^2.
\label{eq:consistency-loss}
\end{equation}
While this equivalence is a direct algebraic consequence of Tweedie's formula, its conceptual importance is central to our work: it provides the formal justification for our score-space objective.

\begin{table*}[t]
\centering
\resizebox{0.8\textwidth}{!}{
\centering
\begin{tabular}{l|ccccc|ccc|c}
\hline
\multicolumn{1}{c|}{{}} & \multicolumn{5}{c|}{Single Object}                                                                             & \multicolumn{3}{c|}{Multi Objects} & \multicolumn{1}{c}{}                                                  \\
\multicolumn{1}{c|}{Models}                        & Property      & Shape         & Texture & Action & \begin{tabular}[c]{@{}c@{}}Complex \end{tabular} & Concat & Relation      & \begin{tabular}[c]{@{}c@{}}Complex\end{tabular} & Avg \\ \hline
SD1.5~\cite{rombach2022high}                                         & 55.0            & 38.8          & 33.8    & 23.1   & 36.9                                                        & 23.1   & 24.4          & 36.3     & 33.9                                                 \\
SDXL~\cite{podell2023sdxl}                                          & 60.0            & 56.9          & 71.3    & 47.5   & 58.1                                                        & 39.4   & 35.0            & 47.5    & 52.0                                                  \\
PixArt~\cite{chen2023pixart}                                       & 49.4          & 58.8          & 76.9    & 56.3   & 63.1                                                        & 35.6   & 30.0            & 48.1         &52.3                                             \\
SD3.0~\cite{esser2024scaling}                                        & 49.4          & 76.3          & 53.1    & 71.9   & 65.0                                                          & 55.0     & 51.2          & 70.0            & 61.5                                            \\
FLUX~\cite{blackforestlabs2024flux}                                     & 58.1          & 71.9          & 47.5    & 52.5   & 60.0                                                         & 55.0     & 48.1          & 70.3                 &57.9                                     \\ \hline
SynGen~\cite{rassin2023linguistic}                                       & 61.3          & 59.4          & 54.4    & 33.8   & 50.6                                                        & 30.6   & 33.1          & 29.4                   & 44.1                                   \\ \hline
LMD~\cite{lian2023llm}                                          & 23.8          & 35.6          & 27.5    & 23.8   & 35.6                                                        & 33.1   & 34.4          & 33.1   & 30.9                                                   \\
RPG~\cite{yang2024mastering}                                          & 33.8          & 54.4          & 66.3    & 31.9   & 37.5                                                        & 21.9   & 15.6          & 29.4   & 36.4                                                   \\
ELLA~\cite{hu2024ella}                                         & 31.3          & 61.6          & 64.4    & 43.1   & 66.3                                                        & 42.5   & 50.6          & 51.9                & 51.5                                      \\
R2F (SD3)~\cite{park2024rare}                                    & \underbar{89.4} & \underbar{79.4}    & \underbar{81.9}    & \underbar{80.0}     & \underbar{72.5}         & \underbar{70.0}     & \underbar{58.8}          & \underbar{73.8} & \underbar{75.7}\\
\hline
Ours (SD3) & \textbf{96.9} & \textbf{89.4} & \textbf{87.5} & \textbf{85.6} & \textbf{80.0} & \textbf{82.5} & \textbf{65.6} & \textbf{85.0} & \textbf{84.1}

\\
\hline
\end{tabular}
}
\caption{Text-to-image alignment performances in the RareBench with other baselines with GPT-4o based evaluation. Best values are denoted with \textbf{bold}, second-best with \underline{underlined}.}
\label{tab:rarebench_summary}
\end{table*}

\subsection{Closed-Form Adaptive Coefficient}
\label{sec:closed-form adaptive coefficient}

We extend CFG by replacing its single conditional score with a dynamic interpolation between the target and auxiliary anchor scores, leading to the blended score formulation in Eq.~\eqref{eq:blended_score}. 
We then determine the optimal blending weight $\gamma_t^*$ by minimizing our score-space alignment loss $\mathcal{L}(\gamma_t)$ from Eq.~\eqref{eq:consistency-loss}. 
Solving $\nabla_{\gamma_t}\mathcal{L}(\gamma_t)=0$ yields the closed-form solution:
\begin{equation}
\begin{aligned}
&\gamma_t^*(x_t) 
= \\& \frac{1-w}{w} \cdot 
\frac{\langle s_\theta(x_t,\tilde{c}_T)-s_\theta(x_t),\, s_\theta(x_t,\tilde{c}_A)-s_\theta(x_t,\tilde{c}_T)\rangle}
{\|s_\theta(x_t,\tilde{c}_A)-s_\theta(x_t,\tilde{c}_T)\|_2^2}.
\label{eq:final}
\end{aligned}
\end{equation}
This adaptive coefficient dynamically modulates the contribution of the auxiliary anchor at each timestep, ensuring the guidance direction remains maximally aligned with the target score while still benefiting from anchor-driven stabilization. 
Derivation is provided in Appendix~\ref{sec:derivation} and the full algorithm is provided in Appendix~\ref{sec:full_alg}.

Theoretically, this step-wise projection can be viewed as an optimal adaptation process under log-concave target distributions. The following proposition formalizes that adaptive projection yields a tighter squared 2-Wasserstein $W_2^2$~\cite{wasserstein} bound than any fixed interpolation.

\noindent \textbf{Proposition 1.}
(Extension to the log-concave case.) \textit{Consider two distributions: $q_{\gamma_t}$, which uses a fixed coefficient $\gamma_t$ at timestep $t$, and $q_{\mathrm{proj}}$, which adaptively projects $\gamma_t^*$ to minimize local score error.
Suppose $p_T$ is $k$-strongly log-concave and satisfies the transport-information inequality~\cite{OTTO2000361, bruno2023diffusion}}:
\begin{equation}
W_2^2(q, p_T) \leq \frac{1}{k^2} J(q \| p_T),
\end{equation}
\textit{where $k > 0$ is the strong log-concavity constant, and $J(q\|p) := \mathbb{E}_{q}\!\left[\|\nabla \log q - \nabla \log p\|_2^2\right]$ denotes the Fisher divergence~\cite{OTTO2000361}. Then}
\begin{equation}
\begin{aligned}
\forall \gamma_t \in \mathbb{R}, \quad W_2^2\!\left(q_{\mathrm{proj}},p_T\right) 
&\;\le\; \tfrac{1}{k^2}\,J(q_{\mathrm{proj}}\|p_T)\\
&\;\le\; \tfrac{1}{k^2}\,J(q_{\gamma_t}\|p_T).
\end{aligned}
\end{equation}
\textit{Hence, in the idealized log-concave case, pointwise projection leads to a smaller Fisher divergence $J$ and correspondingly tighter squared 2-Wasserstein bound in compared to any fixed interpolation strategy.}

\noindent \textit{Proof.} 
The complete proof is available in Appendix~\ref{sec:theoretical_extension}.

\section{Experiments}

\paragraph{Implementation Details.}
For the rare image generation, we build upon SD3.0 and use $T=50$ sampling steps with a fixed random seed of 42. 
The classifier-free guidance scale $w=7.0$ follows the default configuration of the baseline model. 
Evaluation is conducted on the RareBench~\cite{park2024rare} benchmark, which assesses rare-concept generation quality across eight categories: five single-attribute cases (property, shape, texture, action, complex) and three multi-attribute cases (concatenation, relation, complex).

For image editing, we build upon the FlowEdit framework~\cite{kulikov2024flowedit}, which employs an inversion-free ODE formulation for structure-preserving, text-guided editing.
Experiments are conducted on the FlowEdit benchmark dataset, which comprises over 70 real images from DIV2K~\cite{Ignatov_2018_ECCV_Workshops} and royalty-free online sources, paired with approximately 250 manually curated source–target text prompts.
We set the classifier-free guidance scale to $w=20.5$ for all image editing experiments, while keeping all other hyperparameters identical to those used in the original FlowEdit configuration.
All implementations are based on PyTorch 2.8.0 and executed on a single NVIDIA RTX 6000 Blackwell GPU.

\noindent \textbf{Baselines.}
We compare our method against ten representative baselines across both rare concept generation and image editing tasks. For rare concept generation, we evaluate three groups of approaches: (1) pre-trained T2I diffusion models, including SD1.5~\cite{rombach2022high}, SDXL~\cite{podell2023sdxl}, PixArt-$\alpha$~\cite{chen2023pixart}, SD3.0~\cite{esser2024scaling}, and FLUX~\cite{blackforestlabs2024flux}; (2) a linguistic binding method, SynGen~\cite{rassin2023linguistic}; and (3) LLM-grounded diffusion models, LMD~\cite{lian2023llm}, RPG~\cite{yang2024mastering}, ELLA~\cite{hu2024ella}, and R2F~\cite{park2024rare}.

For image editing experiments, we compare our method with four representative diffusion- and flow-based approaches: SDEdit~\cite{meng2021sdedit}, ODE Inversion~\cite{rout2024semantic}, iRFDS~\cite{yang2024text}, and FlowEdit~\cite{kulikov2024flowedit}.
All baselines are reproduced following the FlowEdit experiments configuration, ensuring consistent classifier-free guidance and noise scheduling.

\begin{figure*}[t!]
    \centering
    \includegraphics[width=0.9\linewidth]{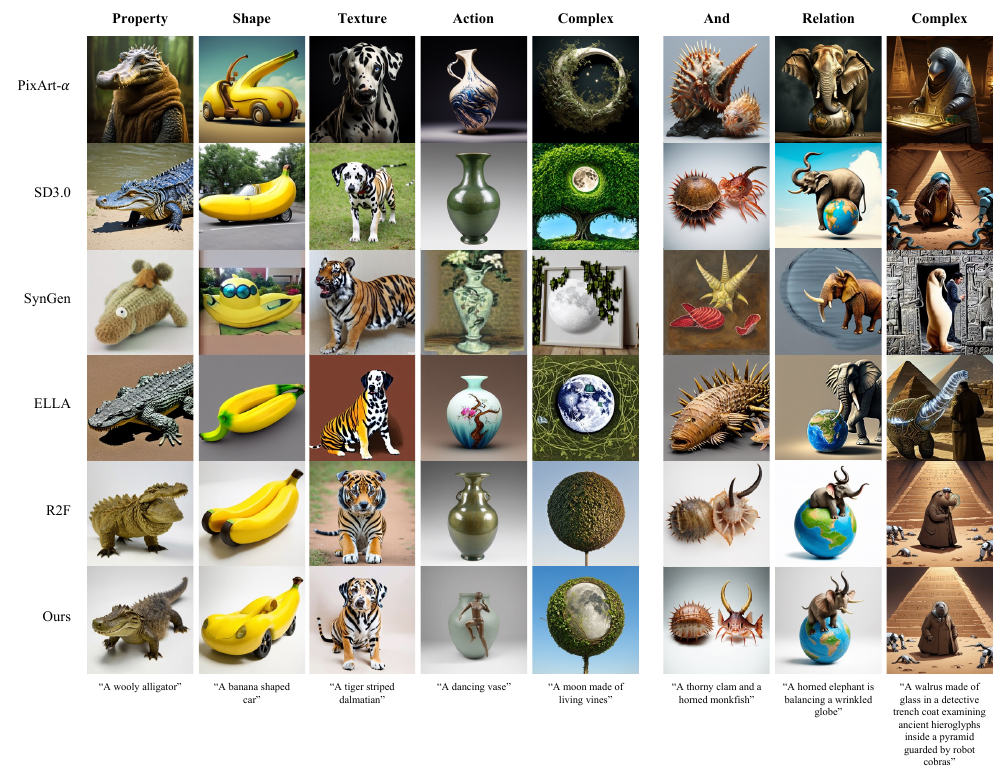}
    \caption{Qualitative comparison with state-of-the-art diffusion models on RareBench. All models are executed with the same random seed. Our method achieves stronger text-to-image alignment without additional training.}
    \label{fig:qualitative_results}
\end{figure*}

\begin{figure*}
    \centering
    \includegraphics[width=0.75\linewidth]{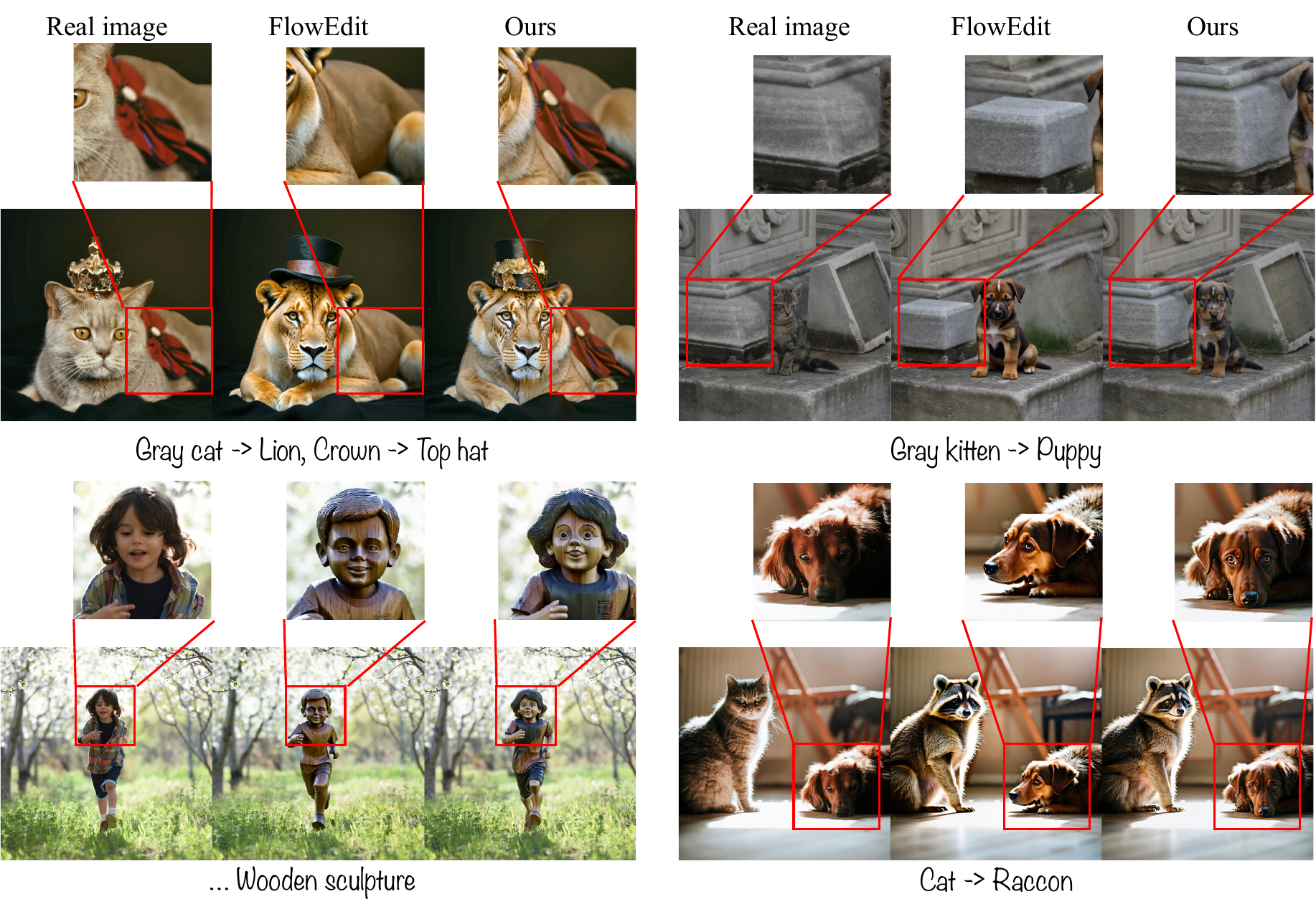}
    \caption{Qualitative comparison of image editing results using FlowEdit~\cite{kulikov2024flowedit} and our method. All edits are performed with the same random seed. Compared to FlowEdit, our approach better preserves source content while faithfully applying the instructed edits.}
    \label{fig:flowedit_qualitative}
\end{figure*}

\noindent \textbf{Experimental Setup.}
We evaluate our approach on two representative tasks: rare concept generation and image editing.
Given a user prompt, we employ auxiliary anchors differently depending on the evaluation setting.
For rare concept generation, each prompt is decomposed into a set of concept pairs $\{c^i\}_{i=1}^m$, where each $c^i=(c_T^i,c_A^i)$ comprises a target (rare) concept $c_T^i$ and its semantically aligned anchor (frequent) counterpart $c_A^i$ (predicted by an LLM), with $m$ denoting the number of target (rare) concepts.
Two complete prompts are then reconstructed:
\begin{align}
\tilde{c}_T &= \text{Reconstruct}(\{c_T^i\}_{i=1}^m), \\
\tilde{c}_A &= \text{Reconstruct}(\{c_A^i\}_{i=1}^m),
\end{align}
where $\text{Reconstruct}(\cdot)$ reinserts the concept set into the original sentence structure. 
In this case, the auxiliary anchor corresponds to the frequent prompt $\tilde{c}_A$.  
For image editing, we instead define the auxiliary anchor as the source (unedited) prompt, allowing the edited target prompt to be guided toward semantic fidelity without compromising structural consistency with the original image.
LLM instructions for rare concept generation and their effects are provided in the Appendix~\ref{sec:llm_instruction}.

\subsection{Main Results of AAPB}

\noindent \textbf{Toy Example.}
To illustrate the effect of adaptive interpolation under controlled conditions, we train a conditional diffusion model on a mixture of two 2D Gaussian distributions: an \emph{auxiliary anchor} $\mathcal{N}((0,-6), I)$ and a \emph{target-prior} $\mathcal{N}((0,3), 1.5I)$ identical to the true target distribution.
This mixture, visualized in Fig.~\ref{fig:toyexample}(a), serves as the training data, where the anchor samples constitute $80\%$ and the target-prior samples $20\%$ of the dataset.
We further include an unconditional branch for $10\%$ of the data following the standard classifier-free guidance scheme~\cite{ho2022classifier}.
After training on these distributions, we generate samples using two interpolation strategies: Fig.~\ref{fig:toyexample}(b) fixed $\gamma_t$ interpolation $p_{\text{lerp}}(x\mid \tilde{c}_A, \tilde{c}_T; \gamma_t=0.8)$ and Fig.~\ref{fig:toyexample}(c) adaptive interpolation $p(x\mid \tilde{c}_A,\tilde{c}_T;\gamma_t^*)$, where $\gamma_t^*$ is dynamically optimized at each denoising step.
The resulting samples are then compared to the true target data distribution $\mathcal{N}((0,3), 1.5I)$, with Fig.~\ref{fig:toyexample}(d) reporting the 2-Wasserstein distance as a function of $\gamma_t$.
While fixed interpolation achieves its minimum around $\gamma_t \approx 0.8$, the adaptive method consistently attains a lower distance, confirming the benefit of step-wise optimization.
This toy example validates that adaptive blending dynamically balances anchor stability and target fidelity. This property is essential for extending the approach to real image generation and editing tasks.

\noindent \textbf{RareBench.} Tab.~\ref{tab:rarebench_summary} compares our method with pre-trained diffusion models (SD1.5, SDXL, PixArt, SD3.0, FLUX), a linguistic binding model (SynGen), and LLM-grounded models (LMD, RPG, ELLA, R2F) on RareBench.
We evaluate text-to-image alignment with GPT-4o, as introduced by~\cite{park2024rare}.
Our approach achieves the highest T2I alignment across all categories, covering both single-object (property, shape, texture, action, complex) and multi-object (concatenation, relation, complex) prompts.
As illustrated in Fig.~\ref{fig:qualitative_results}, our model produces visually stable and prompt-aligned images, whereas R2F often exhibits compositional artifacts, and both SynGen and ELLA struggle with rare attribute binding.
Overall, we obtain an average score of 84.1, outperforming R2F by 8.4 points and substantially surpassing SD3.0 and PixArt, with large gains in property (+7.5), shape (+10.0), and multi-object relation (+6.8). These results highlight the robustness of our rare-to-frequent alignment strategy in boosting semantic faithfulness and structural realism.
Further results with various diffusion baselines, user studies, and image quality evaluations are presented in Appendix~\ref{sec:various_baselines},~\ref{sec:user_study}, and~\ref{sec:Quantitative Image Quality Analysis}, respectively.

\begin{table}[t!]
\centering
\resizebox{1\linewidth}{!}{
\begin{tabular}{c|c|cccc}
\hline
           & CLIP-T $\uparrow$ & CLIP-I $\uparrow$ & LPIPS $\downarrow$ & DINO $\uparrow$  & DreamSim $\downarrow$ \\ \hline
SDEdit 0.2~\cite{meng2021sdedit} & 0.330   & \underbar{0.885}  & 0.251 & 0.634 & \underbar{0.213}    \\
SDEdit 0.4 & 0.340   & 0.854  & 0.316 & 0.564 & 0.273    \\
ODE Inv~\cite{rout2024semantic}    & 0.337  & 0.813  & 0.318 & 0.549 & 0.326    \\
iRFDS~\cite{yang2024text}      & 0.335  & 0.822  & 0.376 & 0.534 & 0.327    \\
FlowEdit~\cite{kulikov2024flowedit}   & \textbf{0.344}  & 0.872  & \underbar{0.181} & \underbar{0.719} & 0.259    \\ \hline
Ours       & \underbar{0.341}  & \textbf{0.905}  & \textbf{0.155} & \textbf{0.814} & \textbf{0.155}    \\ \hline
\end{tabular}
}
\caption{Quantitative results on FlowEdit dataset. CLIP-T$\uparrow$ evaluates text adherence, while CLIP-I$\uparrow$, LPIPS$\downarrow$, DINO$\uparrow$, and DreamSim$\downarrow$ assess structure preservation. Our method achieves superior performance on structure preservation while maintaining strong text alignment. }
\label{table:image_editing}
\end{table}

\noindent \textbf{FlowEdit.} Fig.~\ref{fig:flowedit_qualitative} and Tab.~\ref{table:image_editing} present qualitative and quantitative comparisons on the FlowEdit dataset. 
Our method achieves comparable text alignment to FlowEdit while substantially improving structure preservation metrics (CLIP-I~\cite{radford2021learning}, LPIPS~\cite{zhang2018unreasonable}, DINO~\cite{caron2021emerging}, and DreamSim~\cite{fu2023dreamsim}). 
Qualitatively, our model better preserves the spatial and textural integrity of the original image while faithfully executing the instructed edits, avoiding over-modification or semantic drift frequently observed in FlowEdit. 
Quantitatively, we achieve the highest CLIP-I (0.905) and DINO (0.814) scores, alongside the lowest LPIPS (0.155) and DreamSim (0.155), demonstrating robust content retention and edit precision. 
These results highlight that our approach maintains fine-grained realism and faithful edit alignment. 
Additional qualitative results are provided in Appendix~\ref{sec:more_visualization_results}.

\subsection{Ablation Studies}

\noindent \textbf{Adaptive vs. Fixed Interpolation Scale.} 
To validate the necessity of our adaptive coefficient, we evaluate fixed $\gamma_t$ values ranging from 0.0 to 1.0 in increments of 0.1 for both rare concept generation and image editing tasks.
As shown in Fig.~\ref{fig:gamma_t_experiments}, performance follows a distinct convex trend: intermediate coefficients (around $\gamma_t\approx0.3$–$0.5$) achieve peak accuracy, while both extremes sharply degrade generation quality.
This behavior mirrors our toy example (Fig.~\ref{fig:toyexample}), where the same convex pattern emerges in the idealized 2D Gaussian setting.
This trend appears consistently across both toy and real image settings, indicating that the anchor–target trade-off generalizes beyond idealized conditions. Consequently, no single fixed coefficient can maintain optimal alignment throughout the diffusion process.
In contrast, our AAPB method adaptively adjusts $\gamma_t$ throughout denoising, consistently outperforming all fixed-coefficient baselines and the R2F model without any manual tuning. 

For image editing, we further examine the trade-off between structure preservation (CLIP-I) and text alignment (CLIP-T) across various fixed $\gamma_t$ values, as illustrated in Fig.~\ref{fig:clip-tradeoff}.
When $\gamma_t$ is 0.0, edits overly distort the source image—improving text consistency but severely compromising structural fidelity.
Conversely, large $\gamma_t$ values retain structure but weaken text alignment, indicating that fixed interpolation cannot balance both fidelity dimensions simultaneously.
Our adaptive strategy occupies a Pareto-favorable region in this trade-off space, preserving competitive image structure while achieving stronger text alignment than all fixed configurations and existing editing baselines.
Lastly, we provide a comprehensive analysis of varying classifier-free guidance scales $w$ and the behavior of the adaptive coefficient $\gamma_t^{*}$ in Appendix~\ref{sec:varying_cfg_guidance} and Appendix~\ref{sec:adaptive_coefficient_progress}, respectively.

\begin{figure}[t!]
    \centering
    \includegraphics[width=1\linewidth]{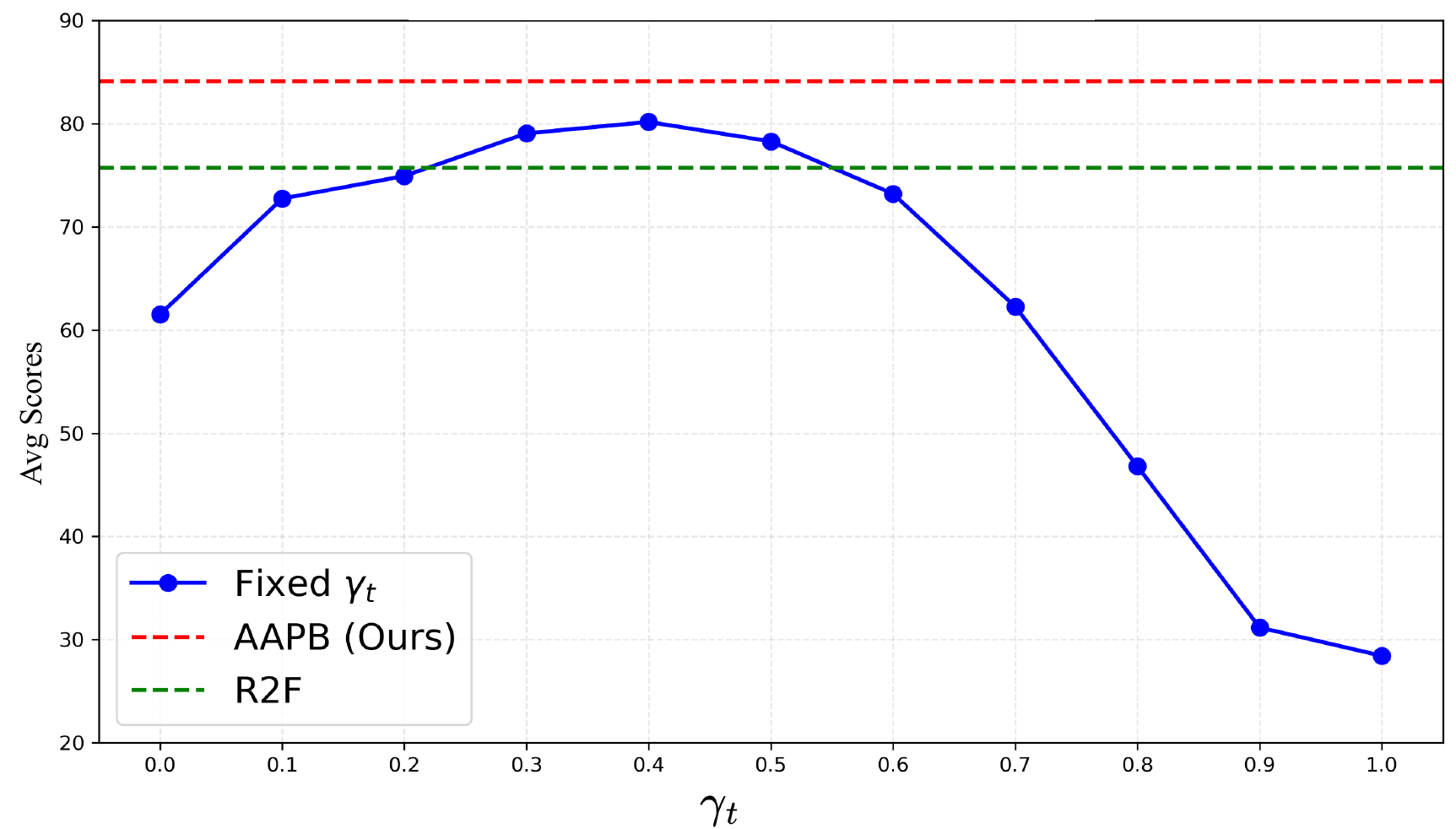}
    \caption{Comparison between fixed $\gamma_t$, R2F~\cite{park2024rare}, and our adaptive coefficient on RareBench. The blue line denotes fixed $\gamma_t$, the green line denotes R2F, and the red line represents our adaptive approach, which consistently outperforms the fixed baseline.}
    \label{fig:gamma_t_experiments}
\end{figure}

\begin{figure}[t!]
    \centering
    \includegraphics[width=1\linewidth]{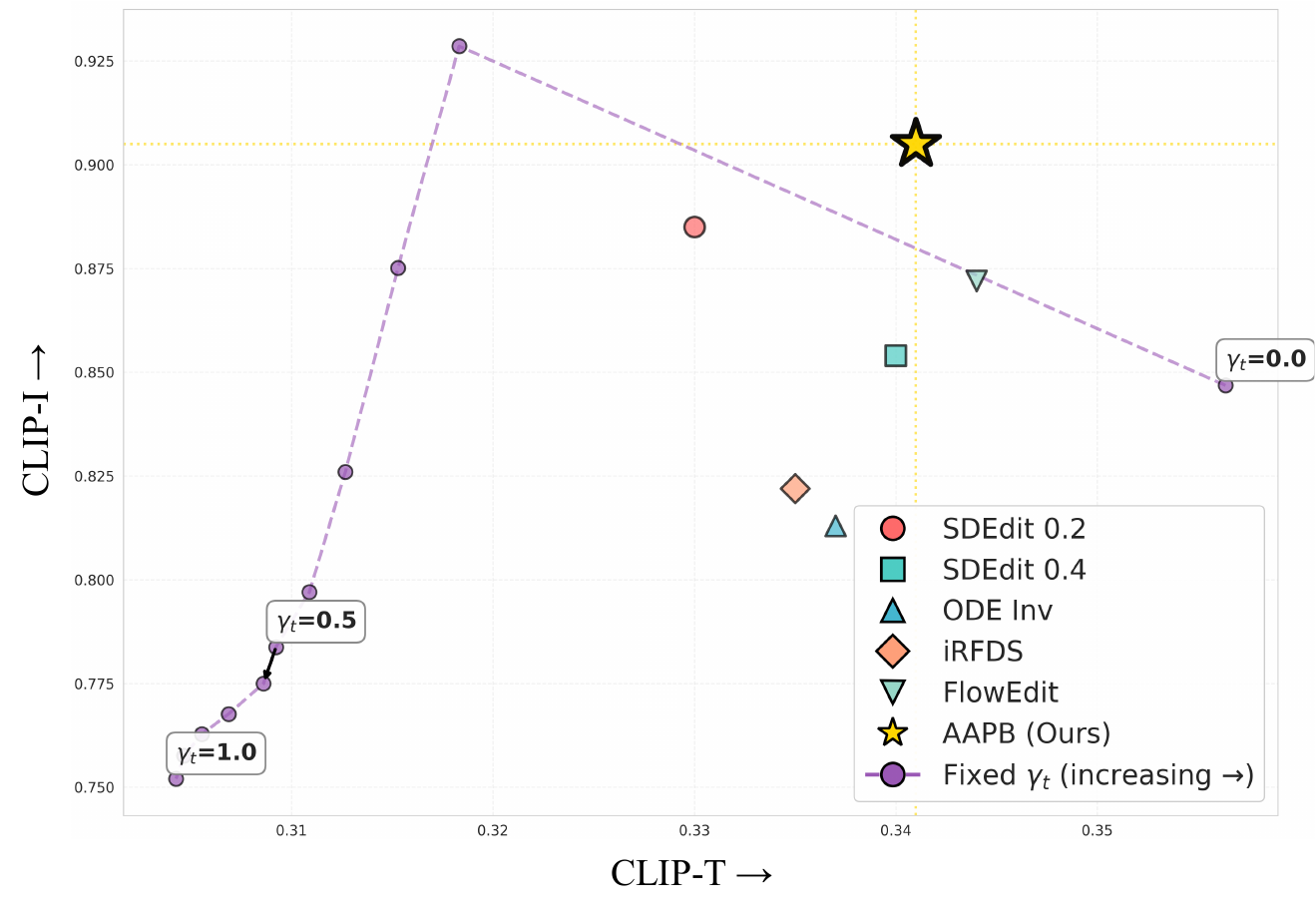}
    \caption{Trade-off between structure preservation (CLIP-I) and text alignment (CLIP-T) across different fixed $\gamma_t$ values in FlowEdit dataset. 
    Our method consistently resides in a Pareto-favorable region, balancing structural preservation and text alignment more effectively than fixed baselines.
    }
    \label{fig:clip-tradeoff}
\end{figure}

\noindent \textbf{Anchor Sensitivity Analysis.}
We assess robustness to anchor prompt quality using four anchor generation strategies on RareBench (Tab.~\ref{tab:anchor}), each representing a different level of semantic guidance: (i) human-annotated anchors from R2F reflecting expert prior knowledge, (ii) arbitrary anchors randomly selected without regard to target concepts, (iii) a minimal ``objects'' replacement that replaces rare concepts with the generic term object (e.g., ``hairy frog'' $\rightarrow$ ``hairy object''),  and (iv–v) LLaMA3 (LLaMA3-8B-Instruct~\cite{grattafiori2024llama}) and GPT-4o generated anchors that utilize LLM-generated content to follow given instructions (Appendix~\ref{sec:llm_instruction}).  
Across all strategies, AAPB consistently outperforms R2F, showing strong robustness even under weak or noisy semantic anchors. 
Notably, GPT-4o-generated anchors achieve the best performance, surpassing both human annotations and LLaMA3. 
This validates two key findings: (1) AAPB remains effective even with minimal or imperfect semantic anchors, as the adaptive coefficient (Eq.~\eqref{eq:final}) automatically balances anchor influence based on alignment strength; and (2) state-of-the-art LLMs can automate anchor selection, surpassing manual annotation while enabling practical and scalable deployment without expert intervention.
Additional analysis is provided in Appendix~\ref{sec:anchor-quality}.

\begin{table}[t!]
\centering
\resizebox{1\linewidth}{!}{
\begin{tabular}{c|ccccc|c}
\hline
Models          & Property & Shape & Texture & Action & Complex & Avg \\ \hline
SD3.0           & 49.4     & 76.3  & 53.1    & 71.9   & 65.0  & 63.1  \\ \hline
Human Generated (R2F) & 79.8    & 68.8  & 76.3    & 78.5   & -     & 75.9   \\
Human Generated (Ours) &\textbf{ 93.1 }    & \textbf{74.4}  & \textbf{83.7 }   & \textbf{79.2}   & -     & \textbf{82.6 }  \\ \hdashline

Arbitrary (R2F)          & 57.5     & 80.0  & \textbf{63.1}    & 70.0   & 69.4  & 68.0 \\
Arbitrary (Ours)         & \textbf{65.6}     & \textbf{83.7}  & 62.5    & \textbf{73.1}   & \textbf{70.0}  & \textbf{71.0} \\ \hdashline

``objects'' (R2F)     & \textbf{90.0}          & \textbf{78.1}       & 80.6         & 78.8         &  80.6       & 81.6      \\
``objects'' (Ours)      & 89.4     & 74.4     & \textbf{87.5}    & \textbf{80.6 }  & \textbf{83.7}  & \textbf{83.1}  \\ \hdashline
LLaMA3 (R2F)          & 81.9     & 77.1 & 76.3    & 78.8   & 67.7  & 76.4 \\
LLaMA3 (Ours)         & \textbf{82.5}     & \textbf{77.5}  & \textbf{80.6}    & \textbf{85.6}   & \textbf{78.8}  & \textbf{81.0} \\ \hdashline
GPT-4o (R2F)             & 89.4     & 79.4  & 81.9    & 80.0   & 72.5  & 80.6  \\ 
GPT-4o (Ours)        & \textbf{96.9}     & \textbf{89.4}  & \textbf{87.5}    &\textbf{ 85.6}   & \textbf{80.0}  & \textbf{87.9}  \\ \hline
\end{tabular}
}
\caption{Experiments with different anchor prompt strategies. Our method consistently outperforms R2F~\cite{park2024rare} regardless of the anchor prompts.}
\label{tab:anchor}
\end{table}

\section{Conclusion}

We proposed Adaptive Auxiliary Prompt Blending (AAPB), a unified and training-free framework that stabilizes the diffusion process in low-density regions.
By adaptively balancing the influence between target and anchor prompts, AAPB achieves both semantic faithfulness and structural consistency.
Grounded in Tweedie's identity, we derived a closed-form adaptive coefficient that provides a principled mechanism for prompt blending.
Experiments on RareBench and FlowEdit confirm that AAPB consistently improves rare concept generation and zero-shot editing.
These findings highlight adaptive score-space modulation as a promising direction for future controllable diffusion generation and editing systems.

\section*{Acknowledgments}
This was partly supported by the Institute of Information \& Communications Technology Planning \& Evaluation (IITP) grant funded by the \grantsponsor{Korean government (MSIT)} (No. \grantnumber{1}{RS-2020-II201373}, Artificial Intelligence Graduate School Program(Hanyang University)) and the Institute of Information \& Communications Technology Planning \& Evaluation (IITP) grant funded by the \grantsponsor{Korea government (MSIT)} (No. \grantnumber{2}{RS-2025-02219062}, Self-training framework for VLM-based defect detection and explanation model in manufacturing process).

{
    \small
    \bibliographystyle{ieeenat_fullname}
    \bibliography{main}
}

\clearpage
\setcounter{page}{1}
\maketitlesupplementary

\section{Theoretical Extension: Log-Concave Setting.}
\label{sec:theoretical_extension}
To provide theoretical insight into why our adaptive coefficient outperforms fixed interpolation, we analyze the idealized case where the target distribution satisfies log-concavity and a transport-information inequality~\cite{OTTO2000361, Robbins1992}.
Under these conditions, we formally establish that pointwise adaptive projection yields a provably tighter upper bound on the squared Wasserstein-2 distance compared to any fixed coefficient strategy.

\begin{proposition}[Extension to the log-concave case]
Consider two distributions: $q_{\gamma_t}$, which uses a fixed coefficient $\gamma_t$ at timestep $t$, and $q_{\mathrm{proj}}$, which adaptively projects $\gamma_t^*$ to minimize local score error.
Suppose $p_T$ is $k$-strongly log-concave and satisfies the transport-information inequality~\cite{OTTO2000361, bruno2023diffusion}.
\begin{equation}
W_2^2(q, p_T) \leq \frac{1}{k^2} J(q \| p_T),
\end{equation}
\textit{where $k > 0$ is the strong log-concavity constant, and $J(q\|p) := \mathbb{E}_{q}\!\left[\|\nabla \log q - \nabla \log p\|_2^2\right]$ denotes the Fisher divergence~\cite{OTTO2000361}. Then}
\begin{equation}
\begin{aligned}
\forall \gamma_t \in \mathbb{R}, \quad W_2^2\!\left(q_{\mathrm{proj}},p_T\right) 
&\;\le\; \tfrac{1}{k^2}\,J(q_{\mathrm{proj}}\|p_T)\\
&\;\le\; \tfrac{1}{k^2}\,J(q_{\gamma_t}\|p_T).
\end{aligned}
\end{equation}
Hence, in the idealized log-concave case, pointwise projection leads to a smaller Fisher divergence $J$ and correspondingly tighter squared 2-Wasserstein bound compared to any fixed interpolation strategy.

\end{proposition}
\noindent \textit{Remark.} While this result relies on the assumption of global log-concavity—which natural image distributions typically do not satisfy—it provides a theoretical motivation for minimizing the local score error pointwise.

\begin{proof}
The Fisher divergence between a distribution $q$ (at time $t$) and the target $p_T$ is defined as the expected squared error between their score functions:
\begin{equation}
J(q \| p_T)=\mathbb{E}_{x \sim q}\left[\left\|\nabla_x \log q(x)-\nabla_x \log p_T(x)\right\|_2^2\right].
\end{equation}
In our framework, we aim to approximate the target score $\nabla \log p_T$ using the blended model $\tilde{s}_\theta$. For clarity, let $s_T(x_t) := \nabla_{x_t} \log p_T(x_t)$ denote the true target score (approximated by the learned rare score in practice).

Consider the comparison at a specific timestep $t$.
The score for our projection model is $s_{\text{proj}}(x_t) = \tilde{s}_\theta(x_t; w, \gamma_t^*(x_t))$.
The score for the baseline model with a scalar coefficient is given by $s_{\gamma_t}(x_t) = \tilde{s}_\theta(x_t; w, \gamma_t)$, where $\gamma_t \in \mathbb{R}$ denotes the coefficient applied at timestep $t$ and is fixed across all timesteps.

By definition, our adaptive coefficient $\gamma_t^*(x_t)$ is the solution that minimizes the squared error to the target score $s_T(x_t)$ at each spatial point $x_t$:
\begin{equation}
\gamma_t^*(x_t) = \underset{{\gamma_t} \in \mathbb{R}}{\arg \min} \left\| \tilde{s}_\theta(x_t ; w, \gamma_t) - s_T(x_t) \right\|_2^2.
\end{equation}
Since $\gamma_t^*(x_t)$ is the minimizer over $\mathbb{R}$ for every individual $x_t$, the error is guaranteed to be lower than or equal to the error yielded by any scalar $\gamma_t$:
\begin{equation}
\forall \gamma_t \in \mathbb{R}, \quad \left\|s_{\text{proj}}(x_t) - s_T(x_t)\right\|_2^2 \leq \left\|s_{\gamma_t}(x_t) - s_T(x_t)\right\|_2^2.
\end{equation}
Since this pointwise inequality holds for all $x_t$ in the support of $q_t$, taking the expectation over $x_t \sim q_t$ preserves the inequality:
\begin{equation}
\begin{aligned}
\mathbb{E}_{x_t \sim q_t}&\left[\left\|s_{\text{proj}}(x_t) - s_T(x_t)\right\|_2^2\right] \\
&\leq \mathbb{E}_{x_t \sim q_t}\left[\left\|s_{\gamma_t}(x_t) - s_T(x_t)\right\|_2^2\right].
\end{aligned}
\end{equation}
This implies that our method achieves a minimized instantaneous Fisher divergence compared to any scalar choice $\gamma_t$:
\begin{equation}
J(q_{\mathrm{proj}} \| p_T) \leq J(q_{\gamma_t} \| p_T).
\end{equation}
Via the transport-information inequality ($W_2^2 \leq \frac{1}{k^2}J$), this lower Fisher divergence implies a tighter upper bound on the squared Wasserstein-2 distance to the target distribution at each timestep.
\end{proof}

\section{Derivation of Closed-Form Adaptive Coefficient}
\label{sec:derivation}

We provide a detailed derivation of the closed-form solution for our adaptive coefficient $\gamma_t^*$ presented in Eq.~\eqref{eq:final} of the main paper.

\begin{proposition}[Optimal Adaptive Coefficient]
\label{prop:optimal_gamma}
For a given $(x_t, t)$ with $w > 0$, assume that $\|s_\theta(x_t,\tilde{c}_A) - s_\theta(x_t,\tilde{c}_T)\|_2^2 > 0$. 
Then the optimal adaptive coefficient $\gamma_t^*$ that minimizes the score-space alignment loss $\mathcal{L}(\gamma_t) = \|\tilde{s}_\theta(x_t;w,\gamma_t) - s_\theta(x_t,\tilde{c}_T)\|_2^2$ has the closed-form solution:
\begin{equation}
\begin{aligned}
& \gamma_t^*(x_t) = \\ & \frac{1-w}{w} \cdot \frac{\langle s_\theta(x_t,\tilde{c}_T) - s_\theta(x_t), 
s_\theta(x_t,\tilde{c}_A) - s_\theta(x_t,\tilde{c}_T) \rangle}
{\|s_\theta(x_t,\tilde{c}_A) - s_\theta(x_t,\tilde{c}_T)\|_2^2}.
\end{aligned}
\end{equation}
\end{proposition}

\begin{proof}
Recall our blended score function from Eq.~\eqref{eq:blended_score}:
\begin{equation}
\begin{aligned}
\tilde{s}_\theta(x_t; &w,\gamma_t) = s_\theta(x_t) \\
&+ w\left((1-\gamma_t)s_\theta(x_t,\tilde{c}_T) + \gamma_t s_\theta(x_t,\tilde{c}_A) - s_\theta(x_t)\right),
\end{aligned}
\end{equation}
where $s_\theta(x_t)$ is the unconditional score, $s_\theta(x_t,\tilde{c}_T)$ is the target-conditioned score, and $s_\theta(x_t,\tilde{c}_A)$ is the auxiliary anchor-conditioned score.

Our objective is to minimize the score-space alignment loss from Eq.~\eqref{eq:consistency-loss}:
\begin{equation}
\mathcal{L}(\gamma_t) = \|\tilde{s}_\theta(x_t;w,\gamma_t) - s_\theta(x_t,\tilde{c}_T)\|_2^2.
\end{equation}
First, we simplify the blended score by distributing $w$:
\begin{multline}
\tilde{s}_\theta(x_t;w,\gamma_t) = s_\theta(x_t) + w(1-\gamma_t)s_\theta(x_t,\tilde{c}_T) \\
+ w\gamma_t s_\theta(x_t,\tilde{c}_A) - ws_\theta(x_t) \\
= (1-w)s_\theta(x_t) + w(1-\gamma_t)s_\theta(x_t,\tilde{c}_T) + w\gamma_t s_\theta(x_t,\tilde{c}_A).
\end{multline}
Substituting into the loss:
\begin{align}
\mathcal{L}(\gamma_t) 
&= \|(1-w)s_\theta(x_t) + w(1-\gamma_t)s_\theta(x_t,\tilde{c}_T) \nonumber\\
&\quad + w\gamma_t s_\theta(x_t,\tilde{c}_A) - s_\theta(x_t,\tilde{c}_T)\|_2^2 \nonumber\\
&= \|(1-w)s_\theta(x_t) - \gamma_t w s_\theta(x_t,\tilde{c}_T) + w\gamma_t s_\theta(x_t,\tilde{c}_A) \nonumber\\
&\quad + ws_\theta(x_t,\tilde{c}_T) - s_\theta(x_t,\tilde{c}_T)\|_2^2 \nonumber\\
&= \|(1-w)s_\theta(x_t) + (w-1)s_\theta(x_t,\tilde{c}_T) \nonumber\\
&\quad + w\gamma_t\left(s_\theta(x_t,\tilde{c}_A) - s_\theta(x_t,\tilde{c}_T)\right)\|_2^2.
\end{align}
For notational convenience, define:
\begin{align}
\mathbf{r} &:= (1-w)(s_\theta(x_t) - s_\theta(x_t,\tilde{c}_T)), \\
\mathbf{d} &:= s_\theta(x_t,\tilde{c}_A) - s_\theta(x_t,\tilde{c}_T).
\end{align}
Then the loss becomes:
\begin{equation}
\mathcal{L}(\gamma_t) = \|\mathbf{r} + w\gamma_t \mathbf{d}\|_2^2.
\end{equation}
\begin{align}
\mathcal{L}(\gamma_t) &= \langle \mathbf{r} + w\gamma_t \mathbf{d}, \mathbf{r} + w\gamma_t \mathbf{d} \rangle \nonumber\\
&= \|\mathbf{r}\|_2^2 + 2w\gamma_t \langle \mathbf{r}, \mathbf{d} \rangle + w^2\gamma_t^2 \|\mathbf{d}\|_2^2.
\end{align}
This is a quadratic function in $\gamma_t$.
Taking the derivative with respect to $\gamma_t$:
\begin{equation}
\frac{\partial \mathcal{L}}{\partial \gamma_t} = 2w\langle \mathbf{r}, \mathbf{d} \rangle + 2w^2\gamma_t \|\mathbf{d}\|_2^2.
\end{equation}
Setting the derivative to zero:
\begin{align}
2w\langle \mathbf{r}, \mathbf{d} \rangle + 2w^2\gamma_t^* \|\mathbf{d}\|_2^2 &= 0 \nonumber\\
\gamma_t^* &= -\frac{\langle \mathbf{r}, \mathbf{d} \rangle}{w\|\mathbf{d}\|_2^2}.
\end{align}
Recall that $\mathbf{r} = (1-w)(s_\theta(x_t) - s_\theta(x_t,\tilde{c}_T))$ and $\mathbf{d} = s_\theta(x_t,\tilde{c}_A) - s_\theta(x_t,\tilde{c}_T)$:
\begin{align}
&\gamma_t^*(x_t) \nonumber \\&= \frac{\langle (w-1)(s_\theta(x_t) - s_\theta(x_t,\tilde{c}_T)), s_\theta(x_t,\tilde{c}_A) - s_\theta(x_t,\tilde{c}_T) \rangle}{w\|s_\theta(x_t,\tilde{c}_A) - s_\theta(x_t,\tilde{c}_T)\|_2^2} \nonumber\\
&= \frac{(w-1)\langle s_\theta(x_t) - s_\theta(x_t,\tilde{c}_T), s_\theta(x_t,\tilde{c}_A) - s_\theta(x_t,\tilde{c}_T) \rangle}{w\|s_\theta(x_t,\tilde{c}_A) - s_\theta(x_t,\tilde{c}_T)\|_2^2} \nonumber\\
&= \frac{w-1}{w} \cdot \frac{\langle s_\theta(x_t) - s_\theta(x_t,\tilde{c}_T), s_\theta(x_t,\tilde{c}_A) - s_\theta(x_t,\tilde{c}_T) \rangle}{\|s_\theta(x_t,\tilde{c}_A) - s_\theta(x_t,\tilde{c}_T)\|_2^2} \nonumber\\
 &= \frac{1-w}{w} \cdot \frac{\langle s_\theta(x_t,\tilde{c}_T) - s_\theta(x_t), s_\theta(x_t,\tilde{c}_A) - s_\theta(x_t,\tilde{c}_T) \rangle}{\|s_\theta(x_t,\tilde{c}_A) - s_\theta(x_t,\tilde{c}_T)\|_2^2}.
\end{align}
which matches Eq.~\eqref{eq:final} in the main paper.

To confirm this is a minimum, we check the second derivative:
\begin{equation}
\frac{\partial^2 \mathcal{L}}{\partial \gamma_t^2} = 2w^2 \|\mathbf{d}\|_2^2 = 2w^2 \|s_\theta(x_t,\tilde{c}_A) - s_\theta(x_t,\tilde{c}_T)\|_2^2 > 0,
\end{equation}
confirming that $\gamma_t^*$ is indeed a global minimum of the quadratic loss function.
\end{proof}

\begin{table*}[t]
\centering
\resizebox{0.9\textwidth}{!}{
\begin{tabular}{c|l|ccccc|ccc|c}
\hline
\multicolumn{1}{c|}{} & \multicolumn{1}{c|}{} & \multicolumn{5}{c|}{Single Object} & \multicolumn{3}{c|}{Multi Objects} &  \\ 
& Models & Property & Shape & Texture & Action & Complex & Concat & Relation & Complex & Avg \\ \hline
\multirow{6}{*}{\rotatebox{90}{\centering \footnotesize SDXL}}
& SDXL~\cite{podell2023sdxl} & 60.0 & 56.9 & 71.3 & 47.5 & 58.1 & 39.4 & 35.0 & 47.5 & 52.0 \\
& LMD~\cite{lian2023llm}                                          & 23.8          & 35.6          & 27.5    & 23.8   & 35.6                                                        & 33.1   & 34.4          & 33.1   & 30.9                                                   \\
& RPG~\cite{yang2024mastering}                                          & 33.8          & 54.4          & 66.3    & 31.9   & 37.5                                                        & 21.9   & 15.6          & 29.4   & 36.4                                                   \\
& ELLA~\cite{hu2024ella}                                         & 31.3          & 61.6          & 64.4    & 43.1   & 66.3                                                        & 42.5   & 50.6          & 51.9                & 51.5                                      \\
& R2F (SDXL)~\cite{park2024rare} & 71.3 & 77.5 & 73.8 & 54.4 & 70.6 & 50.6 & 36.0 & 52.8 & 60.9 \\
& Ours (SDXL) & \textbf{85.6} & \textbf{83.1} & \textbf{86.9} & \textbf{70.0} & \textbf{80.0} & \textbf{64.4} & \textbf{55.0} & \textbf{63.8} & \textbf{73.6} \\ \hline

\multirow{3}{*}{\rotatebox{90}{\centering \footnotesize IterComp}} 
& IterComp~\cite{zhang2024itercomp} & 63.8 & 66.9 & 61.3 & 65.6 & 61.9 & 41.3 & 29.4 & 53.1 & 55.4 \\
& R2F (IterComp)~\cite{park2024rare} & 78.1 & \textbf{77.5} & 79.4 & 66.9 & 63.9 & 41.5 & 36.6 & 53.4 & 62.2 \\
& Ours (IterComp) & \textbf{90.6} & \textbf{77.5} & \textbf{88.1} & \textbf{86.3} & \textbf{84.4} & \textbf{75.0} & \textbf{66.3} & \textbf{65.6} & \textbf{79.2} \\ \hline

\multirow{3}{*}{\rotatebox{90}{\centering \footnotesize SD3.0}} 
& SD3.0~\cite{esser2024scaling} & 49.4 & 76.3 & 53.1 & 71.9 & 65.0 & 55.0 & 51.2 & 70.0 & 61.5 \\
& R2F (SD3)~\cite{park2024rare} & 89.4 & 79.4 & 81.9 & 80.0 & 72.5 & 70.0 & 58.8 & 73.8 & 75.7 \\
& Ours (SD3) & \textbf{96.9} & \textbf{89.4} & \textbf{87.5} & \textbf{85.6} & \textbf{80.0} & \textbf{82.5} & \textbf{65.6} & \textbf{85.0} & \textbf{84.1} \\ \hline

\end{tabular}
}
\caption{Text-to-image alignment on RareBench with SDXL, SD3.0, and IterComp, comparing our method integrated into each model. Results demonstrate consistent robustness across diverse pre-trained diffusion backbones.}
\label{tab:sdxl}
\end{table*}

\section{Toy Example.}
\label{sec:toy_example}
We empirically validate our method on a toy problem, generating samples from a 2D Gaussian \emph{target} distribution using interpolation with an \emph{auxiliary anchor} distribution.
This simplified setup highlights the effectiveness of adaptive blending compared to fixed blending strategies.

The target data distribution $p_\text{data}(x\mid \tilde{c}_T)$ is defined as $\mathcal{N}((0, 3), 1.5I)$.
We train a conditional diffusion model on a mixture of two distributions: an \emph{auxiliary anchor} distribution $\mathcal{N}((0, -6), I)$, positioned distant from the target, and a \emph{target-prior} distribution $\mathcal{N}((0, 3), 1.5I)$ that is identical to the target distribution.
This setup represents an idealized scenario in which the target prior provides the same information as the target distribution itself—corresponding to accurate score estimation under sufficient model capacity~\cite{song2020score, um2023don}.
The auxiliary anchor samples constitute $80\%$ of the training data, while the target-prior samples account for the remaining $20\%$.
Additionally, following the standard classifier-free guidance (CFG) training scheme~\cite{ho2022classifier}, we include an unconditional branch for $10\%$ of the training data, where the conditioning input is randomly dropped.

Fig.~\ref{fig:toyexample} shows generated samples under two different $\gamma_t$ strategies. 
(a) visualizes the training distributions, with auxiliary anchor samples (orange) clearly separated from the target-prior samples (purple) that overlap with the target region. 
(b) shows samples generated using fixed $\gamma_t$ interpolation $p_\text{lerp}(x\mid \tilde{c}_A, \tilde{c}_T;\gamma_t=0.8)$, which linearly combines the anchor and target-prior score functions. 
(c) demonstrates our adaptive interpolation $p(x\mid \tilde{c}_A,\tilde{c}_T; \gamma_t^*)$, where $\gamma_t^*$ is dynamically optimized at each denoising step. 
(d) reports 2-Wasserstein distance between the generated distributions and the target $\mathcal{N}((0,3), 1.5I)$ as a function of $\gamma_t$. 
The blue curve shows that fixed interpolation achieves its minimum distance around $\gamma_t \approx 0.8$. 
In contrast, the adaptive method (red dashed line) consistently achieves a lower distance, confirming the benefit of step-wise optimization.

The key insight from this toy example is that the optimal interpolation parameter $\gamma_t^*$ varies across the denoising process. 
Fixed interpolation requires manual tuning and remains suboptimal, while our adaptive approach automatically balances target fidelity and anchor stability at each step, leading to a more accurate approximation of the target distribution.

\section{Full Algorithm}
\label{sec:full_alg}

Let $s_\theta(x_t,t,\cdot)$ denote a pretrained score estimator (either unconditional or conditional) evaluated at state $x_t$ and diffusion/flow timestep $t$.
We write $s_\theta(x_t,t) \equiv s_\theta(x_t,t,\varnothing)$ for the unconditional score, and $s_\theta(x_t,t,\tilde{c})$ for the score conditioned by prompt $\tilde{c}$.
Classifier-Free Guidance (CFG) combines unconditional and conditional scores linearly with a guidance scale $w$.
Following the main text, the blended conditional score is defined as
\begin{equation}
s_\theta(x_t,t,\tilde{c}) \;=\; (1-\gamma_t)\,s_\theta(x_t,t,\tilde{c}_T) \;+\; \gamma_t\,s_\theta(x_t,t,\tilde{c}_A),
\label{eq:blended_conditional}
\end{equation}
and the final guided score is given by
\begin{equation}
\tilde{s}_\theta(x_t;t,w,\gamma_t) \;=\; s_\theta(x_t,t) \;+\; w\,\!\left(s_\theta(x_t,t,\tilde{c}) - s_\theta(x_t,t)\right).
\label{eq:final_guided}
\end{equation}
Using Tweedie’s identity under the Gaussian corruption model, the image-domain denoising loss is equivalent to a score-space $\ell_2$ loss, which yields a closed-form per-timestep optimum $\gamma_t^{\*}$ (see main text, Eq.~\eqref{eq:image_domain_loss} -- \eqref{eq:final}).

\begin{algorithm}[t!]
\begin{algorithmic}[1]
\Require Target prompt $\tilde{c}_T$, Anchor prompt $\tilde{c}_A$, pre-trained score model $s_\theta$, guidance scale $w$, timesteps $T$
\Ensure Generated image $x_0$

\State $x_T \sim \mathcal{N}(0, \mathbf{I})$ 
\For{$t = T$ \textbf{down to} $1$}
    \State $s_u \gets s_\theta(x_t, t)$ \Comment{Unconditional score}
    \State $s_T \gets s_\theta(x_t, t, \tilde{c}_T)$ \Comment{Target-conditioned score}
    \State $s_A \gets s_\theta(x_t, t, \tilde{c}_A)$ \Comment{Anchor-conditioned score}
    
    \Statex
    \State $\gamma_t^* \gets \frac{1-w}{w} \cdot \frac{\langle s_T - s_u, s_A - s_T \rangle}{\|s_A - s_T\|_2^2}$ \Comment{Eq.~\eqref{eq:final}}
    
    \Statex
    \State $\tilde{s}_\theta \gets s_u + w \left( (1-\gamma_t^*)s_T + \gamma_t^*s_A - s_u \right)$
    
    \Statex
    \State $x_{t-1} \gets \text{SamplerStep}(x_t, \tilde{s}_\theta, t)$ 
\EndFor
\State \Return $x_0$
\end{algorithmic}
\caption{Adaptive Auxiliary Prompt Blending (AAPB) Generation}
\label{alg:aapb}
\end{algorithm}

\section{Robustness across Various Diffusion Models.}
\label{sec:various_baselines}
Tab.~\ref{tab:sdxl} shows the robustness across three different pre-trained diffusion models on RareBench, including SDXL, IterComp~\cite{zhang2024itercomp}, and SD3.0.
Our method yields substantial performance improvements compared to R2F across all models, achieving consistent average gains of $+12.7$, $+17.0$, and $+8.4$ on SDXL, IterComp, and SD3.0, respectively.
These results demonstrate that our method generalizes effectively across different architectures and training settings, confirming its robustness and broad applicability.

As illustrated in Fig.~\ref{fig:qual_different_architectures}, our AAPB framework consistently preserves the semantic meaning of the input prompt across diverse pre-trained diffusion backbones, including SDXL, IterComp, and SD3.0.
While baseline models often exhibit incomplete attribute binding or geometric distortion in rare or complex prompts, our method accurately conveys the intended semantics while maintaining visual coherence and structural realism.
This demonstrates that adaptive prompt blending effectively generalizes across architectural variations, allowing the model to maintain semantic faithfulness even under different backbone configurations.

\begin{figure}[t!]
    \centering
    \includegraphics[width=0.8\linewidth]{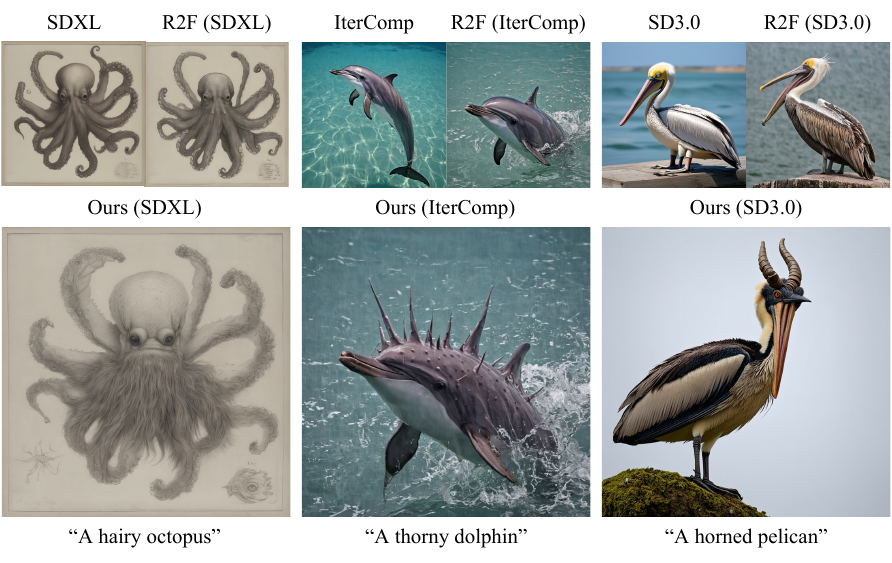}
    \caption{Qualitative comparison on different pre-trained diffusion baselines, SDXL, SD3.0, and IterComp. Comparing our method integrated with these pre-trained models.}
    \label{fig:qual_different_architectures}
\end{figure}

\begin{figure}[t!]
    \centering
    \includegraphics[width=1\linewidth]{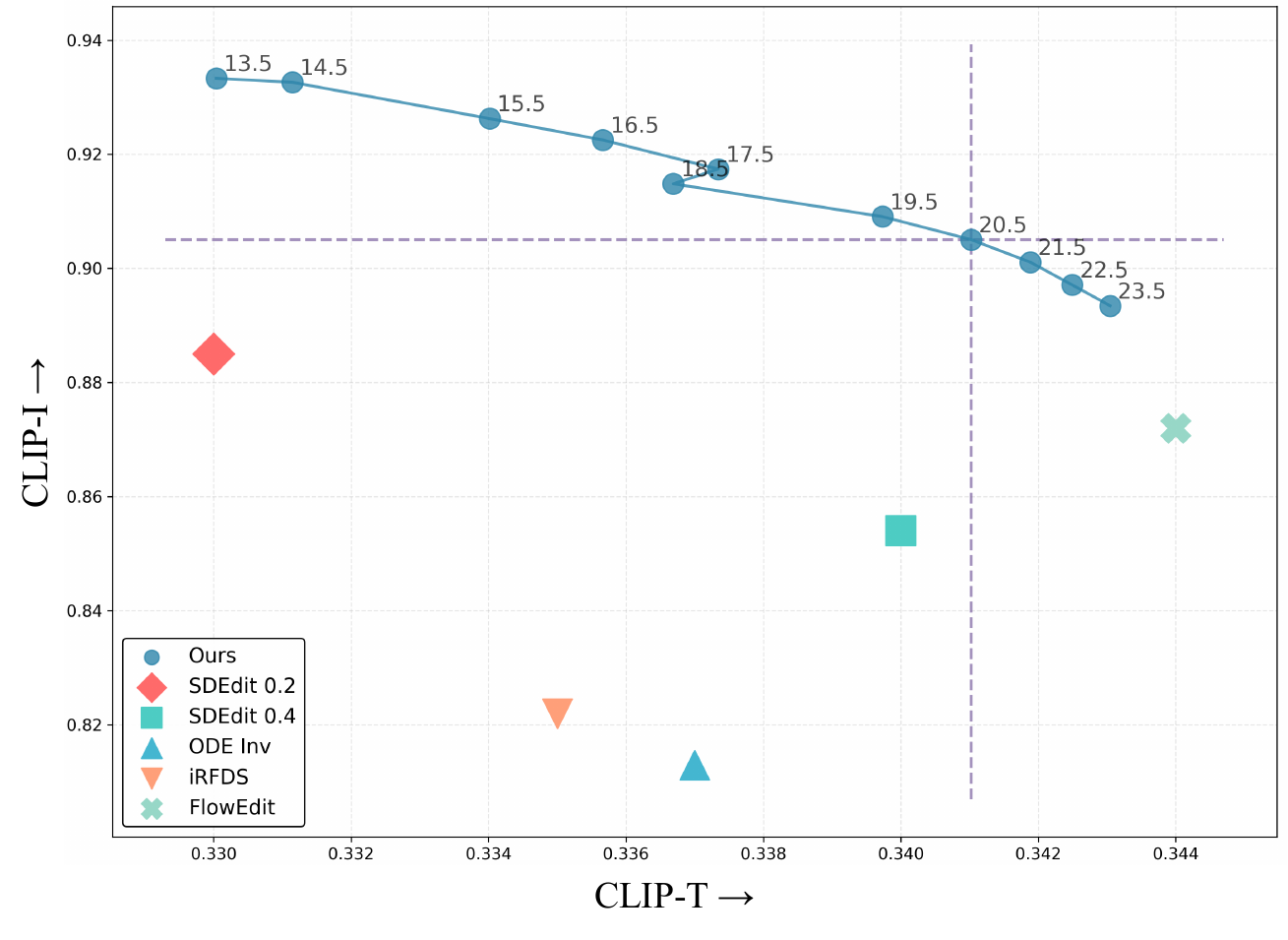}
    \caption{Quantitative comparison on varying classifier-free guidance scale $w$ on the FlowEdit dataset.}
    \label{fig:varying_cfg_flowedit}
\end{figure}

\begin{figure}[t!]
 \centering
 \includegraphics[width=0.8\linewidth]{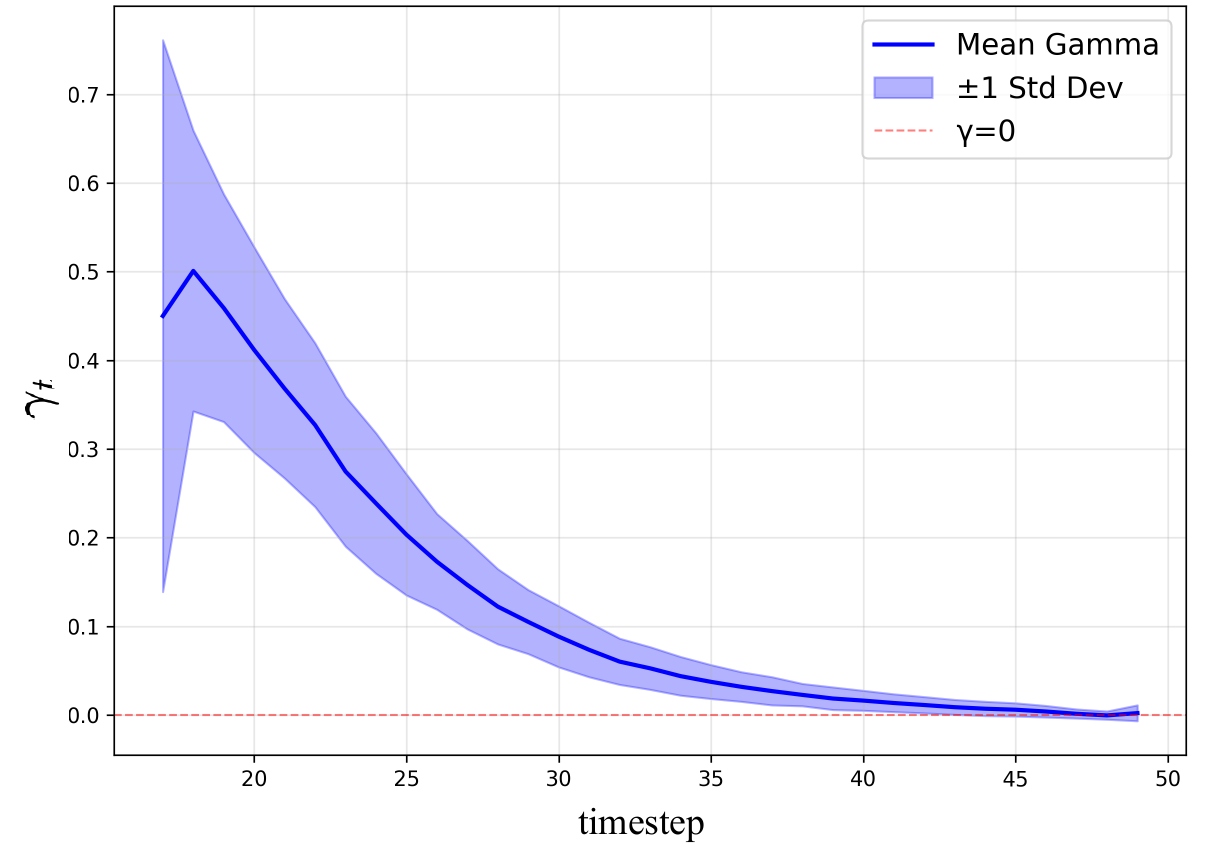}
 \caption{Evolution of the adaptive coefficient $\gamma_t^*$ across diffusion timesteps in image editing (FlowEdit). The coefficient exhibits clear saturation in later steps as structural guidance from the source image becomes dominant.}
 \label{fig:gamma_t_progress}
\end{figure}

\begin{figure}
    \centering
    \includegraphics[width=0.8\linewidth]{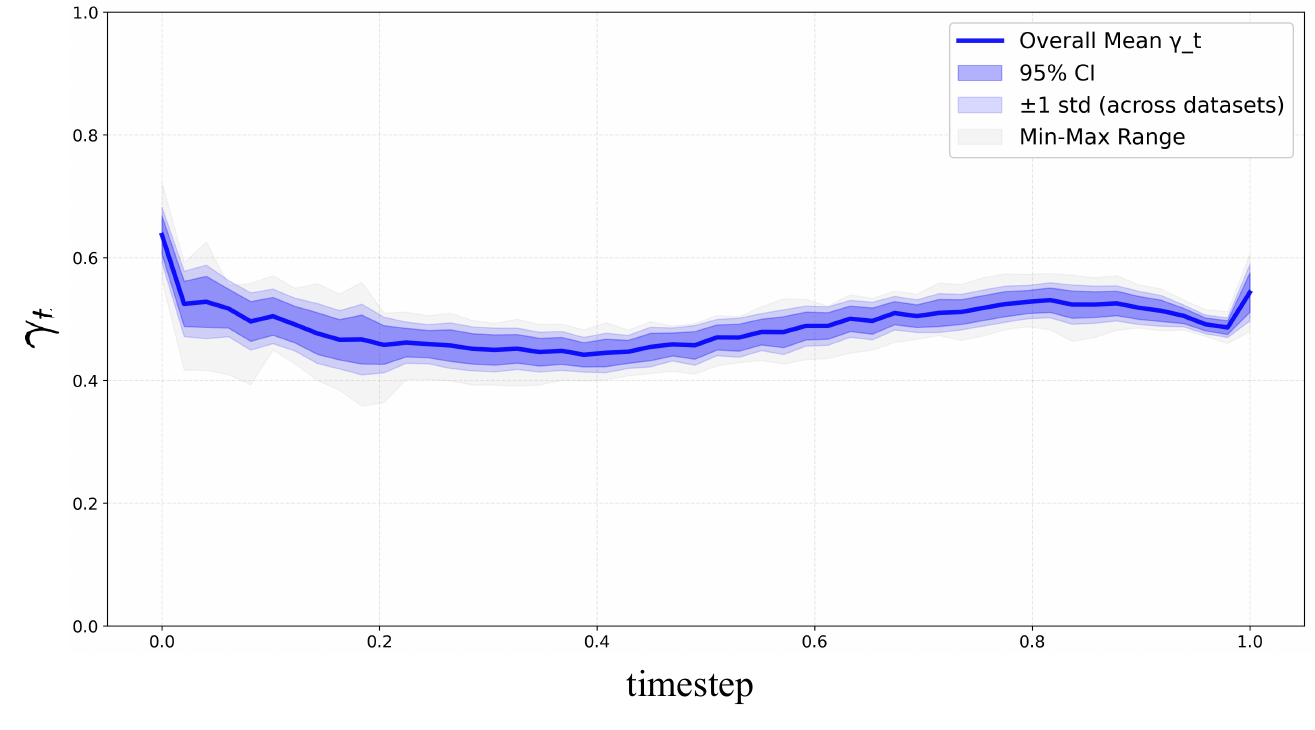}
    \caption{Evolution of the adaptive coefficient $\gamma_t^*$ across diffusion timesteps in rare concept generation (RareBench). Unlike image editing, $\gamma_t^*$ maintains stable values throughout the denoising process, reflecting continuous need for anchor stabilization.}
    \label{fig:rarebench_gamma_t}
\end{figure}

\section{Varying Classifier-Free Guidance Scale}
\label{sec:varying_cfg_guidance}

For rare concept generation, Fig.~\ref{fig:cfg-guidance-scale} shows qualitative results for varying classifier-free guidance scales $w$.
When $w$ is too small (e.g., $w=1.0$ or $2.0$), the generated images lack visual fidelity and semantic alignment with the input prompt.
Conversely, overly large scales (e.g., $w=9.0$ or $10.0$) lead to noticeable artifacts and structural distortions.
We adopt $w=7.0$ as our default setting, which is also the standard guidance scale used in SD3.
Our adaptive coefficient operates on top of this baseline to further refine the balance between fidelity and semantic consistency.

For image editing, Fig.~\ref{fig:varying_cfg_flowedit} presents a quantitative comparison across different classifier-free guidance scales $w$ in terms of text–image alignment (CLIP-T) and image–image fidelity (CLIP-I).
When $w$ is near 13.5, the model exhibits strong structural preservation of the original image but weak semantic alignment with the target prompt.
Conversely, when $w$ becomes near 23.5 semantic consistency improves marginally, yet structural distortions emerge and the overall gain over SDEdit~0.2 remains limited.
We therefore set $w=20.5$ as a balanced configuration that jointly maximizes both CLIP-T and CLIP-I scores, achieving an optimal trade-off between fidelity and semantic accuracy.

\section{Adaptive Coefficient Progress.}
\label{sec:adaptive_coefficient_progress}
Fig.~\ref{fig:gamma_t_progress} presents the mean adaptive coefficient $\gamma_t^*$ over all samples in image editing, showing a shift from stronger anchor (source prompt) reliance in early noisy steps to more target-conditioned guidance in later stages.
Following the FlowEdit configuration, we report $\gamma_t^*$ only after step 17, as the early steps do not reflect meaningful guidance behavior. 
Geometrically, $\gamma_t^*$ can be viewed as the projection of the target score residual $(s_\theta(x_t,\tilde{c}_T) - s_\theta(x_t))$ onto the anchor--target direction. 
As the diffusion progresses, their dot product approaches zero, implying near-orthogonality between target and anchor gradients. 
This indicates that the target score moves toward a self-consistent manifold region requiring less auxiliary correction. 
Hence, $\gamma_t^*$ naturally decays as the model adaptively balances anchor and target influences, producing stable and controllable generations.

In contrast, Fig.~\ref{fig:rarebench_gamma_t} shows that $\gamma_t^*$ in rare concept generation maintains relatively stable values throughout the denoising process (mean $\approx$ 0.5), without the saturation observed in image editing. 
This difference stems from the fundamental distinction in task structure: image editing benefits from strong structural guidance provided by the source image, which progressively dominates the generation process and reduces the need for anchor support in later steps. 
However, rare concept generation operates without such structural constraints---the model must synthesize images purely from textual descriptions. 
Consequently, the anchor continues to provide essential semantic stabilization throughout the entire denoising trajectory, preventing drift toward high-density regions while maintaining consistent guidance toward the target concept.

\section{Additional Analysis for Anchor Sensitivity Analysis with Anchor Quality Metric}
\label{sec:anchor-quality}

Building upon Sec.~\ref{sec:closed-form adaptive coefficient}, we now analyze what constitutes an effective anchor. Although Eq.~\eqref{eq:final} guarantees optimal blending for any given anchor, different anchor construction strategies (Tab.~\ref{tab:anchor}) exhibit surprisingly competitive performance. This observation suggests that the practical effectiveness of an anchor is not determined solely by its directional alignment with the target–unconditional axis, but also by its overall deviation within the score space.

Substituting $\gamma_t^*$ into Eq.~\eqref{eq:consistency-loss} yields the minimum attainable loss at timestep $t$:
\begin{equation}
\label{eq:min-loss-t}
\mathcal{L}_t^*(s_A) = \|\mathbf{r}\|^2 - \frac{\langle \mathbf{r}, \mathbf{d} \rangle^2}{\|\mathbf{d}\|^2},
\end{equation}
where $\mathbf{r} = (1-w)(s_u - s_T)$ is the fixed residual and $\mathbf{d} = s_A - s_T$ is the anchor displacement. 

To understand how anchor choice affects this loss, we decompose the displacement into components parallel and orthogonal to $\mathbf{r}$:
\begin{equation}
\mathbf{d} = \mathbf{d}^{\parallel} + \mathbf{d}^{\perp}, \qquad
\mathbf{d}^{\parallel} = \frac{\langle s_A - s_T,\, s_u - s_T \rangle}{\|s_u - s_T\|^2}(s_u - s_T).
\end{equation}
Since $\langle \mathbf{r}, \mathbf{d}^{\perp} \rangle = 0$, the numerator of Eq.~\eqref{eq:min-loss-t} depends only on $\mathbf{d}^{\parallel}$, while the denominator grows with both components:
\begin{equation}
\mathcal{L}_t^*(s_A) = \|\mathbf{r}\|^2 - \frac{\langle \mathbf{r}, \mathbf{d}^{\parallel} \rangle^2}{\|\mathbf{d}^{\parallel}\|^2 + \|\mathbf{d}^{\perp}\|^2}.
\end{equation}
This reveals a key trade-off: while $\mathbf{d}^{\parallel}$ provides the guidance signal, excessive $\|\mathbf{d}^{\perp}\|$ dilutes its effectiveness by inflating the denominator.

\begin{table}[t!]
\centering
\small
\begin{tabular}{lcccc}
\toprule
\textbf{Strategy} & \textbf{$\|\mathbf{d}^{\parallel}\|$} & \textbf{$\|\mathbf{d}^{\perp}\|$} & \textbf{Total} $\downarrow$ & \textbf{T2I} $\uparrow$ \\
\midrule
Arbitrary  & 42.2 & 47.9 & 90.1 & 71.0 \\
Objects    & 32.5 & 26.6 & 59.1 & 83.1 \\
Human Generated     & 30.2 & 28.4 & 58.6 & 82.6 \\
GPT-4o     & 27.3 & 28.0 & 55.3 & \textbf{87.9} \\
LLaMA3     & 26.8 & 27.0 & \textbf{53.8} & 81.0 \\
\bottomrule
\end{tabular}
\caption{Anchor quality comparison across different anchor generation strategies. Lower displacement indicates anchors closer to the target with fewer artifacts. T2I scores are evaluated using GPT-4o. All correlations are negative.}
\label{tab:anchor_quality}
\end{table}

Motivated by this insight, we propose a simple anchor quality metric:
\begin{equation}
\text{Disp}(s_A) = \|\mathbf{d}^{\parallel}\| + \|\mathbf{d}^{\perp}\|.
\end{equation}
As shown in Tab.~\ref{tab:anchor_quality}, this displacement metric exhibits a consistent trend with T2I alignment performance. Strategies with moderate displacement (e.g., GPT-4o: 55.3, Human Generated: 58.6, Objects: 59.1) achieve strong T2I scores (82.6--87.9), while anchors with excessive displacement (e.g., Arbitrary: 90.1) show significantly degraded performance (71.0), likely due to orthogonal drift that introduces spurious features from distant regions of the score space.

Crucially, however, geometry is not the sole determinant.
While LLaMA3 achieves the lowest displacement (53.8), it trails GPT-4o (55.3) in T2I accuracy (81.0 vs. 87.9).
This suggests that the alignment between anchor and target semantics—beyond geometric displacement alone—also contributes to performance. GPT-4o's superior language understanding may produce anchors that better preserve task-relevant categorical structure (e.g., substituting ``hairy frog'' with ``hairy animal'' rather than generic ``object''), enabling more effective guidance even at comparable displacement levels. However, this does not require the anchor itself to be semantically detailed or rare; rather, the anchor should maintain appropriate semantic alignment with the target's core attributes.

In summary, effective anchors should satisfy two complementary criteria: (1) maintain a well-calibrated geometric displacement from the target score—close enough to avoid drift into spurious modes, yet sufficiently separated to provide stabilization in low-density regions; and (2) preserve basic semantic alignment with task-relevant attributes, even when substituting rare concepts with frequent ones. By ensuring anchors meet these criteria, AAPB's adaptive coefficient $\gamma_t^*$ can reliably modulate their influence across timesteps, enabling even simple, frequent anchors to perform competitively without requiring rich semantic detail.

\begin{table}[t!]
\centering
\resizebox{0.8\linewidth}{!}{
\begin{tabular}{c|ccc}
\hline
\multirow{2}{*}{Models} & \multicolumn{3}{c}{RareBench Multi}  \\
                        & Concat       & Relation       & Complex                \\ \hline
SD3.0                   & 55.0    & 51.2      & 70.0  \\ 
R2F+        & 74.4     & 63.7     & 64.8              \\
Ours (R2F+)          & 83.8     & 83.6     & 80.3               \\ \hline
\end{tabular}
}
\caption{Quantitative comparison on RareBench-Multi between R2F+ and our method built upon the R2F+ baseline.}
\label{tab:r2f_plus}
\end{table}

\begin{table}[t!]
\centering
\resizebox{1\linewidth}{!}{
\begin{tabular}{c|ccc|cc}
\hline
\multirow{2}{*}{Models} & \multicolumn{3}{c|}{Rare Concept} & \multicolumn{2}{c}{Image Editing} \\
                        & SD3       & R2F       & Ours      & FlowEdit          & Ours          \\ \hline
Peak Memory (GB)        & 31.52     & 31.76     & 32.22     & 19.97             & 19.97         \\
GPU Time (sec)          & 27.14     & 26.12     & 37.94     & 7.41              & 7.54          \\ \hline
\end{tabular}
}
\caption{GPU time and memory required to generate an image.}
\label{tab:memory_analysis}
\end{table}

\section{Results on R2F+ Baseline}
\label{sec:r2f+_baseline}

The R2F+~\cite{park2024rare} framework extends R2F to region-controlled generation, where LLMs extract per-object rare–frequent concept pairs with bounding boxes and visual detail levels. 
Our Adaptive Auxiliary Prompt Blending (AAPB) integrates seamlessly into R2F+’s diffusion pipeline without retraining, operating directly on per-timestep denoising trajectories. 
During alternating region guidance, the adaptive coefficient $\gamma_t$ dynamically modulates frequent-anchor influence to preserve both semantic fidelity and local structure across object-wise generations.

Tab.~\ref{tab:r2f_plus} presents a quantitative comparison between R2F+ and our method built upon the same R2F+ baseline. 
Across all RareBench-Multi categories, our framework consistently outperforms the original R2F+ in text–image alignment accuracy, achieving notable improvements of $+9.4$ on the \textit{Concat}, $+19.9$ on the \textit{Relation}, and $+16.5$ on the \textit{Complex}. 
These results demonstrate that our adaptive auxiliary prompt blending mechanism effectively complements the alternating rare–frequent concept scheduling in R2F+.

\begin{figure}[t!]
    \centering
    \includegraphics[width=1\linewidth]{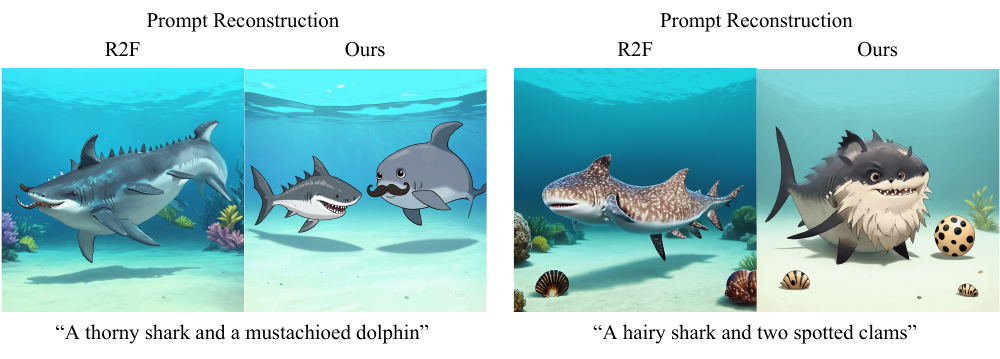}
    \caption{Comparison of prompt reconstruction between R2F and our method. 
    R2F often collapses multiple rare concepts into a single entity, leading to entangled generations. 
    In contrast, our method explicitly preserves each rare concept by directly pairing it with its frequent counterpart, resulting in disentangled and faithful generations.}
    \label{fig:prompt_reconstruction}
\end{figure}

\section{GPU Time and Memory Analysis.}
Tab.~\ref{tab:memory_analysis} reports computational costs measured on an NVIDIA RTX 6000 Blackwell GPU.
For rare concept generation, the peak memory usage remains comparable (33.22GB vs. 31.52GB), indicating that AAPB introduces minimal memory overhead. The longer computation time of AAPB is due to the evaluation of three score functions per timestep (unconditional, target, and anchor) instead of two (unconditional, conditional).
For image editing, the overhead is minimal (7.54s vs. 7.41s) with identical peak memory, as FlowEdit already performs three-branch scores (unconditional, source, and target).
In this case, AAPB simply reuses the source branch as an adaptive anchor, preserving efficiency while enhancing edit faithfulness. 
In return, AAPB delivers a clear performance advantage, significantly surpassing state-of-the-art baselines in both semantic alignment and structural preservation.

\begin{table*}[t!]
\vspace{1mm}
\hrule
\vspace{1mm}

\textbf{$<$Task$>$} \\
Evaluate how well each image matches the given text prompt. Focus on whether the objects in the image and their attributes (e.g., color, shape, texture), spatial arrangement, and actions align with the text.

\vspace{2mm}

\textbf{$<$Rating Scale$>$} \\
5: The image perfectly matches the text prompt with no noticeable mistakes. \\
4: The image matches most of the prompt, with only minor inconsistencies. \\
3: The image reflects the prompt partially, but some important details are missing or incorrect. \\
2: The image shows only a few elements from the prompt, and many key parts are missing or wrong. \\
1: The image does not match the prompt at all or fails to convey the main idea.

\vspace{1mm}
\hrule
\vspace{1mm}
\caption{User study task and rating criteria for evaluating text--image alignment.}
\label{tab:user_study_task}
\end{table*}

\begin{table*}[t]
\centering
\resizebox{0.8\textwidth}{!}{
\centering
\begin{tabular}{l|ccccc|ccc|c}
\hline
\multicolumn{1}{c|}{{}} & \multicolumn{5}{c|}{Single Object}                                                                             & \multicolumn{3}{c|}{Multi Objects} & \multicolumn{1}{c}{}                                                  \\
\multicolumn{1}{c|}{Models}                        & Property      & Shape         & Texture & Action & \begin{tabular}[c]{@{}c@{}}Complex \end{tabular} & Concat & Relation      & \begin{tabular}[c]{@{}c@{}}Complex\end{tabular} & Avg \\ \hline
SD3~\cite{esser2024scaling}  & 46.3  & 58.3 & 46.0 & 43.5 & 53.1 & 40.8 & \underline{46.7} & 60.5 & 49.4 \\
R2F (SD3)~\cite{park2024rare} & \underline{71.9} & \underline{66.2} & \underline{71.4} & \underline{63.7} & \underline{69.3} & \underline{42.6} & 44.6 & \underline{67.3} & \underline{62.1} \\
Ours (SD3) & \textbf{76.7} & \textbf{68.5} & \textbf{71.9} & \textbf{65.0} & \textbf{70.1} & \textbf{59.8} & \textbf{49.7} & \textbf{68.7} & \textbf{66.3}
\\
\hline
\end{tabular}
}
\caption{Qualitative comparison based on our user study. For each prompt, participants were shown multiple images generated by different models in random order and without model names (Fig.~\ref{fig:user_study_screenshot}). Participants selected the preferred result according to the evaluation criteria define in Tab.~\ref{tab:user_study_task}. 
}
\label{tab:user_study_qualitative}
\end{table*}

\section{User Study}
\label{sec:user_study}

We conducted a user study with eleven participants on the RareBench dataset. For each prompt, images generated by Ours, R2F, and SD3 were shown simultaneously to enable direct side-by-side comparison across methods~\cite{sun2023chatgpt}, facilitating a more accurate assessment of semantic alignment and visual consistency. To prevent bias, model names were anonymized and the display order was randomized for every prompt. Participants scored each result using the criteria in Tab.~\ref{tab:user_study_task}, with the $\{1,2,3,4,5\}$ scale linearly mapped to $\{0,25,50,75,100\}$ for analysis (see Fig.~\ref{fig:user_study_screenshot}).

As summarized in Tab.~\ref{tab:user_study_qualitative}, participants frequently noted that our method delivered clearer attribute grounding and reduced compositional inconsistencies, particularly in multi-object scenes. The most pronounced qualitative gains appeared in the Concat and Relation categories, where users consistently favored our results. We attribute this to more explicit entity separation and improved preservation of individual attributes, which enhances visual clarity in complex settings. This finding is consistent with the disentangled prompt reconstruction results in Fig.~\ref{fig:prompt_reconstruction}, demonstrating that our binary alignment strategy better supports multi-entity generation.

\begin{figure}
    \centering
    \includegraphics[width=0.8\linewidth]{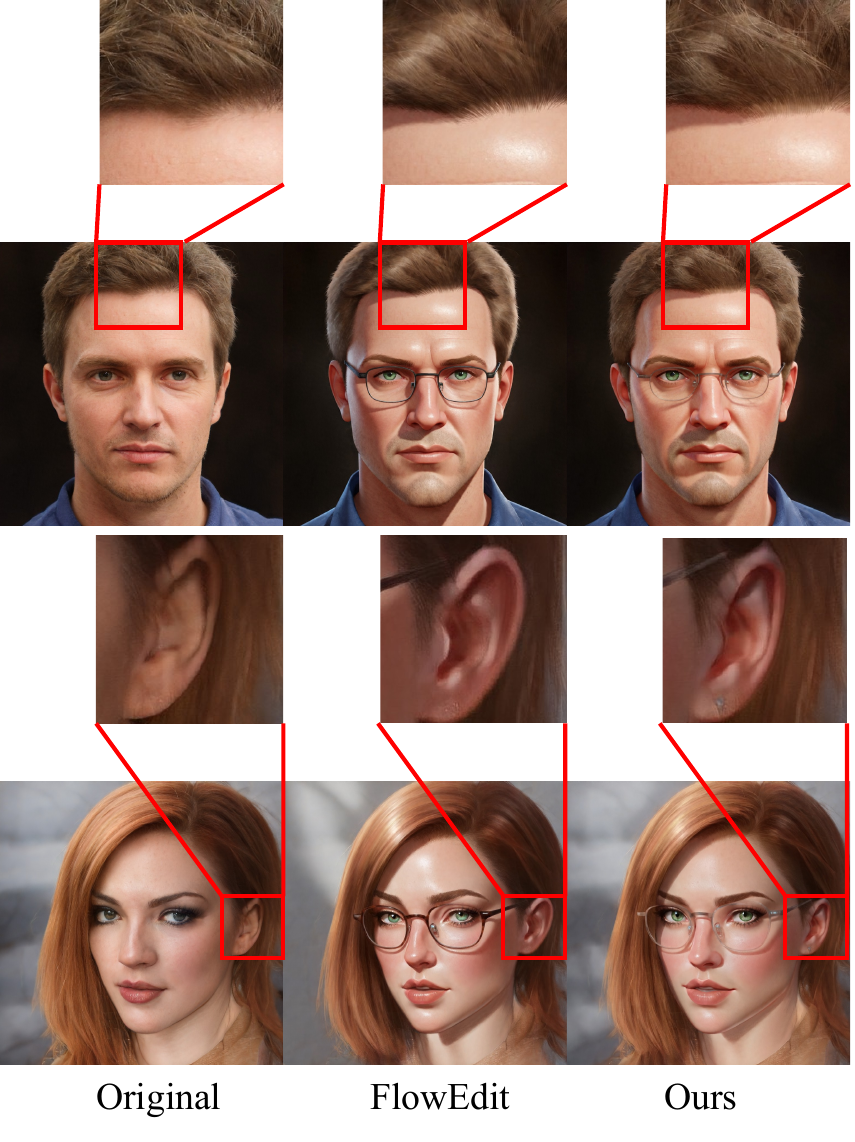}
    \caption{Qualitative results for face-accessory image editing. We extract the source prompt using GPT and augment the target prompt by adding ``glasses'' as an accessory. Our method preserves facial details more faithfully than FlowEdit.}
    \label{fig:FlowEdit_SFHQ}
\end{figure}

\section{LLM Instruction for Rare Concept Generation}
\label{sec:llm_instruction}

Tab.~\ref{tab:aapb_instruction} and Tab.~\ref{tab:aapb_examples_text} detail the full LLM prompt and the in-context examples for AAPB, respectively.
Unlike R2F, which decomposes each rare concept into a \textit{step-wise sequence} with detailed annotations (including an explicit \textit{Visual Detail Level} for heuristic scheduling), our instruction employs a direct binary pairing between rare and frequent concepts.
R2F generates an expanded, hierarchical reasoning trace before concatenating them into the final sequence. 
In contrast, our approach streamlines this process by focusing solely on semantic substitution: it replaces the rare noun phrase with its frequent counterpart in a single pass and directly constructs the rare-frequent pair.

Furthermore, R2F decomposes a prompt into multiple sub-prompts (e.g., ``Two spotted sea creatures'' and ``A hairy aquatic creature''), which are sequentially substituted to form the final sentence. 
In contrast, our formulation enforces a direct rare–frequent pairing (e.g., ``A hairy animal and two spotted objects''), enabling the entire prompt to be reconstructed in a single pass.
In Fig.~\ref{fig:prompt_reconstruction}, we compare the reconstruction strategy of R2F with our binary alignment approach.
Our binary alignment approach resulting in improved disentanglement and fidelity in multi-entity generations.

\section{Image Editing for Face Accessories}

Fig.~\ref{fig:FlowEdit_SFHQ} presents image editing results for adding facial accessories on the SFHQ dataset~\cite{david_beniaguev_2022_SFHQ}. 
We extract the source prompt using GPT and augment the target prompt by appending “glasses’’ as an accessory.
Compared to FlowEdit, our method preserves fine facial details—particularly in the hair and ear regions—while generating the glasses faithfully.

\begin{table}[t!]
\centering
\resizebox{1\linewidth}{!}{
\begin{tabular}{c|ccc}
\hline
Models & LAION-aesthetic & ImageReward & PickScore \\ \hline
SD3.0  & \textbf{6.317±0.190}    & 0.915±0.151      & 21.771±0.462  \\ 
R2F        & 6.215±0.220     & 1.048±0.266     & 22.155±0.239              \\
Ours       & 6.191±0.166     & \textbf{1.123±0.328}     & \textbf{22.177±0.222 }              \\ \hline
\end{tabular}
}
\caption{
Aesthetic and preference evaluation on \textit{RareBench}. 
LAION-aesthetic measures visual appeal, ImageReward models human preference, and PickScore evaluates text–image alignment. 
Our method achieves the best preference and alignment scores.
}
\label{tab:aesthetic}
\end{table}

\section{Quantitative Image Quality Analysis}
\label{sec:Quantitative Image Quality Analysis}

For quantitative image quality evaluation, we employ three widely used metrics: LAION-Aesthetic~\cite{schuhmann2022laion}, PickScore~\cite{kirstain2023pick}, and ImageReward~\cite{xu2023imagereward}.
Tab.~\ref{tab:aesthetic} presents a comparison among AAPB, R2F, and SD3.0.
Our method achieves the highest scores on ImageReward and PickScore, indicating superior human preference alignment and semantic consistency, while SD3.0 attains a slightly higher aesthetic score.
This minor trade-off in visual appeal is acceptable, as AAPB prioritizes rare concept fidelity and prompt-faithful generation over conventional aesthetics—aligning well with the objective of our task.

\begin{figure}[t!]
    \centering
    \includegraphics[width=0.8\linewidth]{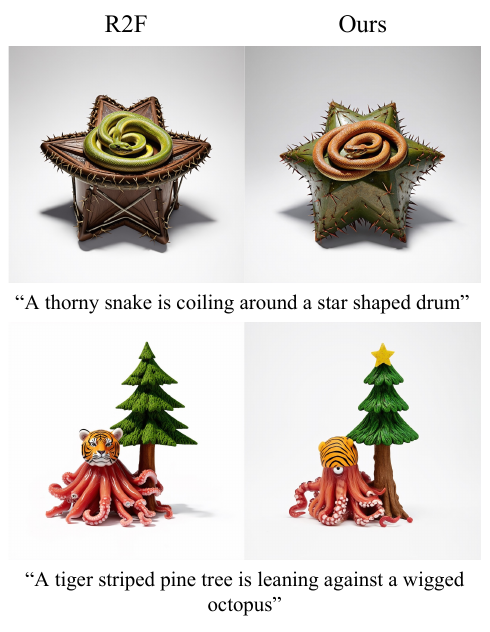}
    \caption{Failure results on RareBench where both ours and R2F exhibit attribute-object mismatches. This behaviors aligns with previously reported limitations in CLIP's ability to bind compositional concpets faithfully~\cite{lewis2022does}.}
    \label{fig:limitation}
\end{figure}

\section{Limitations}
\label{sec:limitations}

Although our approach substantially improves rare-concept alignment, it still inherits the inherent limitations of CLIP-based text encoders, which struggle to preserve compositional bindings when multiple attributes and objects co-exist in a prompt (Fig.~\ref{fig:limitation}). As also reported by Lewis et al.~\cite{lewis2022does}, CLIP often fails to maintain correct attribute–object associations in multi-component scenarios, leading to occasional entanglement or attribute leakage in our generations.
Overcoming this limitation would likely require representation-level advances (e.g., more compositionally robust vision–language encoders) or additional architectural modules explicitly designed for attribute binding, which we leave for future exploration.

\section{More Visualization Results}
\label{sec:more_visualization_results}

Fig.~\ref{fig:varying_seeds_rarebench} presents uncurated generations of our model on RareBench across multiple random seeds.
Most results exhibit strong alignment with the input prompts while preserving naturalness and visual quality.
Furthermore, Fig.~\ref{fig:flowedit_more_results} provides additional qualitative results on the FlowEdit benchmark, where leveraging the blended score yields higher anchor-image fidelity compared to the baseline FlowEdit method.

\clearpage

\begin{table*}[t]
\centering
\small
\begin{tabular}{p{0.98\textwidth}}
\toprule
\textbf{$<$System Prompt$>$} \\
You are a helper language model for a text-to-image generation program that aims to create images based on input text. The program often struggles to accurately generate images when the input text contains rare concepts that are not commonly found in reality. To address this, when a rare concept is identified in the input text, you should replace it with relevant yet more frequent concepts. \\
\\
\textbf{$<$User Prompt$>$} \\
Extract rare concepts from the input text and replace them with relevant yet more frequent ones. Perform the following process step by step: \\
\\
\textbf{a.} Identify and extract any rare concepts from the provided input text. If the text contains one or more rare concepts, extract them all. If there are no rare concepts present, do not extract any concepts. The extracted rare concepts should not overlap. \\
\\
\textbf{b.} Given the rare concepts extracted at a., replace each extracted rare concept with a more frequent concept. Specifically, replace the main noun subject with a more frequent noun subject that is likely to appear in a similar context. Ensure that the replaced frequent noun subject retains the properties of the original main noun subject as much as possible. \\
\\
\textbf{c.} Generate a text sequence that starts from the text with replaced frequent concepts and ends with the text with the original rare concepts. \\
\\
The output should follow the format of the examples below: \\
\midrule
\textbf{$<$In-context Examples$>$} \\
Input: \{USER PROMPT\} \\
Output: \\
\bottomrule
\end{tabular}
\caption{Full LLM instruction for AAPB to generate target-to-anchor concept mappings for rare concept generation.}
\label{tab:aapb_instruction}
\end{table*}

\begin{table*}[t]
\centering
\footnotesize

\begin{tabular}{p{0.98\textwidth}}
\toprule

\textbf{Input:} A banana shaped apple \\
\textbf{Output:} \\
\textbf{Num Rare Concepts:} 1 \\
\quad a. Rare concept: A banana shaped apple \\
\quad b. A banana shaped apple does not exist in reality, while the possibility of a banana shaped red object existing is much higher. Main noun object: apple, Replaced frequent object: red object \\
\quad c. A banana shaped red object BREAK A banana shaped apple \\
\textbf{Final Prompt Sequence:} A banana shaped red object BREAK A banana shaped apple \\
\\
\textbf{Input:} A running dog \\
\textbf{Output:} \\
\textbf{Num Rare Concepts:} 0 \\
\quad No rare concept was found in the given input text. \\
\textbf{Final Prompt Sequence:} A running dog \\
\\
\textbf{Input:} A horned lion and a hairy frog \\
\textbf{Output:} \\
\textbf{Num Rare Concepts:} 2 \\
\quad a. Rare concept: A horned lion \\
\quad b. A horned lion does not exist in reality, while a horned animal does. Main noun subject: lion, Replaced frequent subject: animal \\
\quad c. A horned animal BREAK A horned lion \\
\quad AND \\
\quad a. Rare concept: A hairy frog \\
\quad b. A hairy frog does not exist in reality, while a hairy animal does. Main noun subject: frog, Replaced frequent subject: animal \\
\quad c. A hairy animal BREAK A hairy frog \\
\textbf{Final Prompt Sequence:} A horned animal AND a horned animal BREAK A horned lion AND a hairy frog \\
\\
\textbf{Input:} A horned lion is sitting on a tiger striped rock \\
\textbf{Output:} \\
\textbf{Num Rare Concepts:} 2 \\
\quad a. Rare concept: A horned lion \\
\quad b. A horned lion does not exist in reality, while a horned animal does. Main noun subject: lion, Replaced frequent subject: animal \\
\quad c. A horned animal BREAK A horned lion \\
\quad AND \\
\quad a. Rare concept: a tiger striped rock \\
\quad b. A tiger striped rock does not exist in reality, while a tiger striped object does. Main noun subject: rock, Replaced frequent subject: object \\
\quad c. a tiger striped object BREAK a tiger striped rock \\
\textbf{Final Prompt Sequence:} A horned animal AND a tiger striped object BREAK A horned lion AND a tiger striped rock \\

\bottomrule
\end{tabular}
\caption{In-context examples of the LLM prompt for AAPB.}
\label{tab:aapb_examples_text}
\end{table*}

\begin{figure*}
    \centering
    \includegraphics[width=0.8\linewidth]{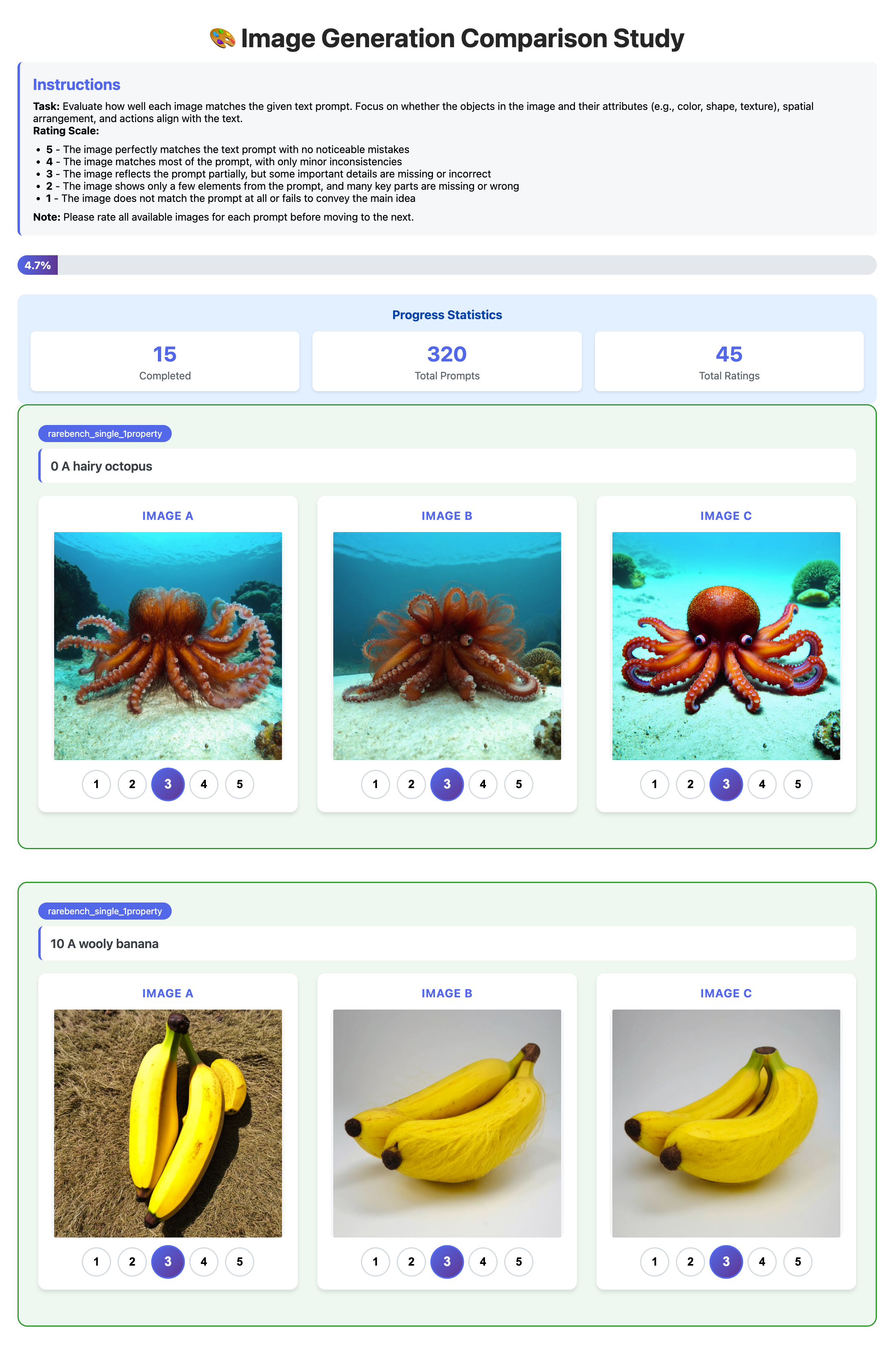}
    \caption{User study interface for evaluating text-to-image alignment. Participants were asked to rate how well each generated image matched the given text prompt based on semantic accuracy and visual consistency. The image order was randomized, and model identities were hidden to avoid bias.}
    \label{fig:user_study_screenshot}
\end{figure*}

\begin{figure*}
    \centering
    \includegraphics[width=1\linewidth]{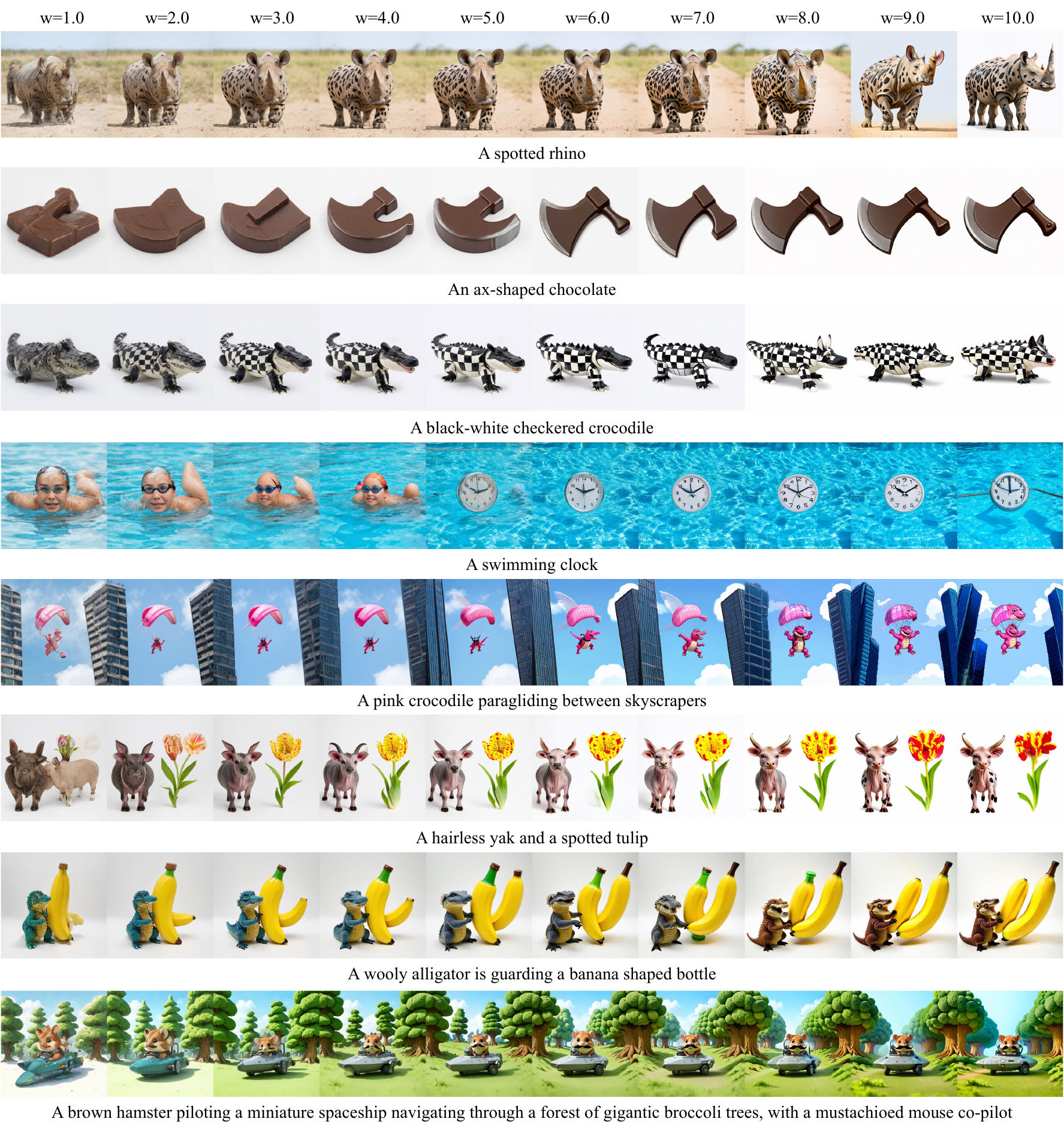}
    \caption{Ablation study with varying classifier-free guidance scale $w$.}
    \label{fig:cfg-guidance-scale}
\end{figure*}

\begin{figure*}
    \centering
    \includegraphics[width=0.95\linewidth]{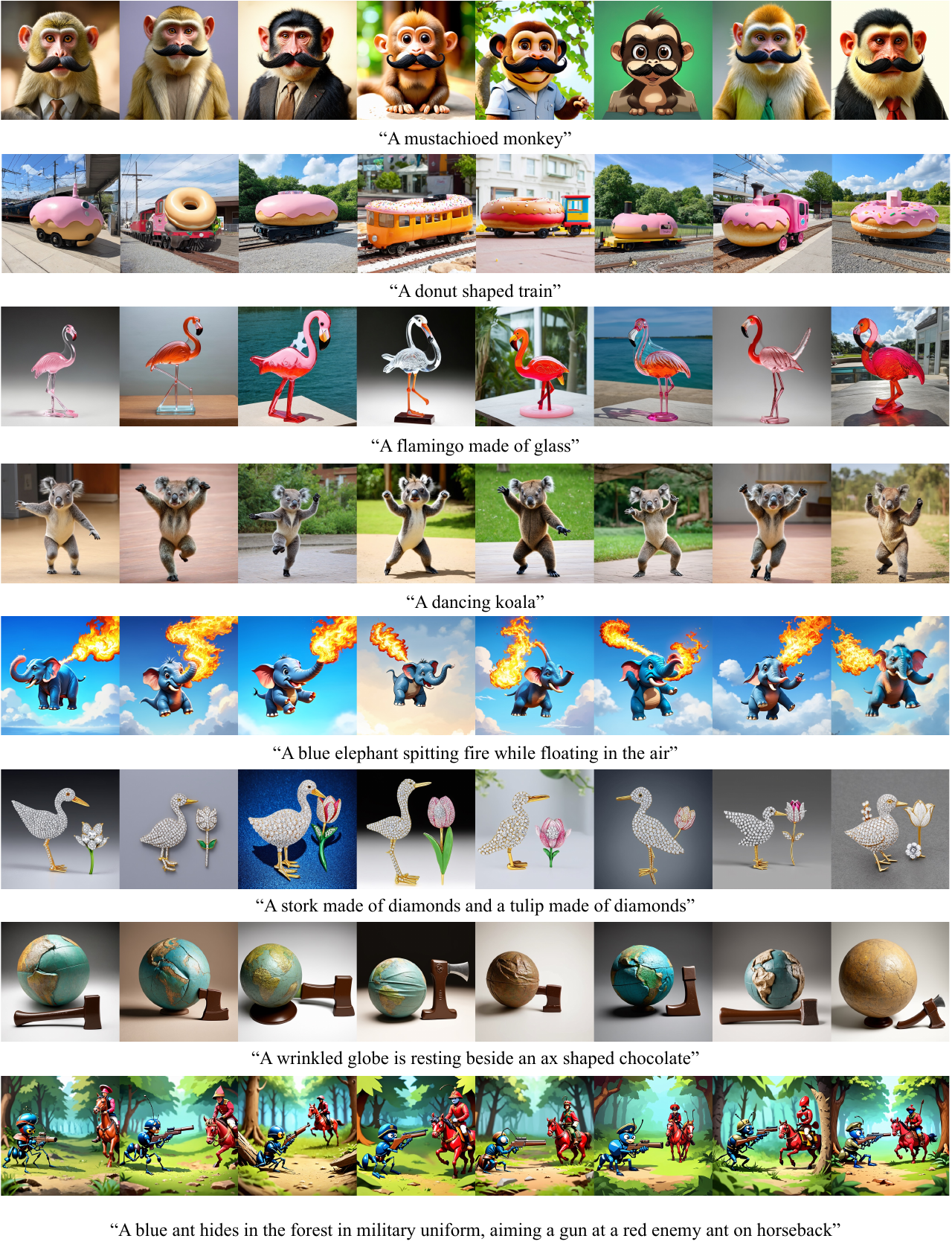}
    \caption{Uncurated qualitative visualizations of our method on RareBench, generated from eight randomly selected prompts across all categories and eight random seeds.}
    \label{fig:varying_seeds_rarebench}
\end{figure*}

\begin{figure*}
    \centering
    \includegraphics[width=1\linewidth]{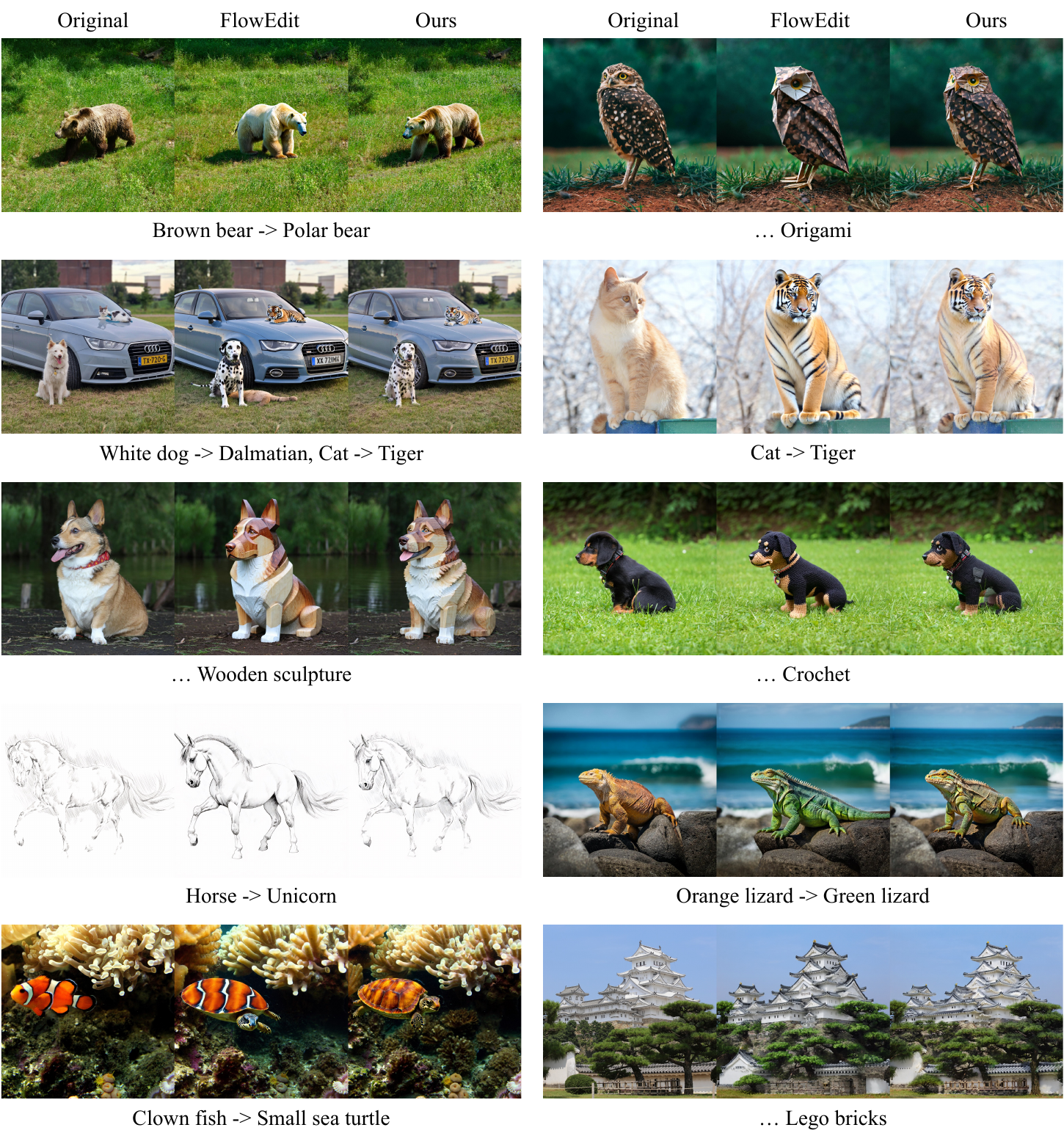}
    \caption{Qualitative results of the image editing.}
    \label{fig:flowedit_more_results}
\end{figure*}

\end{document}